\documentclass{article}
\PassOptionsToPackage{table}{xcolor}
\usepackage{iclr2027_conference,times}

\usepackage{amsmath,amsfonts,bm}

\def\eqref#1{equation~\ref{#1}}

\def\1{\bm{1}}

\DeclareMathAlphabet{\mathsfit}{\encodingdefault}{\sfdefault}{m}{sl}
\SetMathAlphabet{\mathsfit}{bold}{\encodingdefault}{\sfdefault}{bx}{n}

\usepackage{hyperref}
\usepackage{url}

\title{Q-WAM: 4-Bit Quantization of World Action Models with Action-Subspace Protection}

\iclrfinalcopy
\author{Arash Akbari$^{1*\dagger}$ \quad
Arman Akbari$^{1*}$ \quad
Jingwu Luo$^{1}$ \quad
Yuhao Lei$^{1}$ \quad
Yi Gao$^{1}$ \\
\bf Weiwei Chen$^{2}$ \quad
Xuan Zhang$^{1}$ \quad 
Zhenman Fang$^{4}$ \quad 
Geng Yuan$^{3}$ \quad
Yanzhi Wang$^{1}$ \\
$^{1}$Northeastern University \quad $^{2}$EmbodyX
\quad $^{3}$University of Georgia \\
$^{4}$University of Minnesota - Twin Cities
}

\usepackage{multirow} 
\usepackage{booktabs}
\usepackage{hyperref}
\hypersetup{colorlinks=true,linkcolor=red,citecolor=citeblue,urlcolor=citeblue}
\usepackage{url}
\usepackage{graphicx}
\usepackage{amsmath,amssymb}
\usepackage{amsthm}
\usepackage{booktabs}
\usepackage[table]{xcolor}
\definecolor{oursrow}{HTML}{E4ECF7}
\definecolor{citeblue}{HTML}{3399FF}
\newcommand{\subfigref}[2]{\hyperref[#1]{\ref*{#1}#2}}
\usepackage{enumitem}
\usepackage{wrapfig}
\usepackage{xspace}

\newcommand{\z}{\mathbf{z}}

\newcommand{\actionfun}{\mathcal{A}}
\newcommand{\Gt}{\widetilde{G}}
\newcommand{\Wt}{\widetilde{W}}
\newcommand{\xt}{\tilde{x}}
\DeclareMathOperator{\tr}{tr}
\DeclareMathOperator{\diag}{diag}
\DeclareMathOperator{\spn}{span}

\newtheorem{proposition}{Proposition}

\begin{document}

\maketitle
\lhead{Preprint}
{\renewcommand{\thefootnote}{\fnsymbol{footnote}}
\footnotetext[1]{Equal contribution.}
\footnotetext[2]{Corresponding author: \href{mailto:akbari.ara@northeastern.edu}{\texttt{akbari.ara@northeastern.edu}}.}}

\begin{abstract}
World Action Models (WAMs) jointly generate video and robot actions through iterative diffusion and
perform strongly in robotic manipulation. However, their prohibitive compute and memory costs pose substantial deployment challenges. Post-training quantization (PTQ) can reduce these costs, but existing
PTQ methods such as smoothing and rotation are insufficient to maintain the precision of action generation.  To overcome this limitation, we propose \textbf{Q-WAM}, a new 4-bit weight-activation quantization for WAMs that
preserves the actions the model generates. Specifically, we introduce the \textit{Action
Observability Gramian (AOG)}, which measures how much rounding errors in each weighted combination of a layer's input
channels change the final action through all denoising steps. We also develop
\textit{Action-Subspace Protection (ASP)}, which keeps the few most action-sensitive channel combinations in a tiny 16-bit
low-rank branch and quantizes the complementary weights and activations to 4 bits, both as dense matrix multiplications that run efficiently on GPUs.
Finally, to preserve action quality with minimal overhead, we identify the experts that matter most for the generated action by aggregating the AOG-derived action mass across the layers of each expert and apply ASP only to those experts.
We evaluate Q-WAM on three WAMs, both in simulation and in real-world deployment. On the RoboTwin 2.0
benchmark, it reaches 89.6--93.0\% average success rate, within 1.1 percentage points of the 16-bit models, while reducing the memory of
the quantized blocks by \textbf{3.1--3.4}$\times$. Our method outperforms the strongest baseline, SVDQuant, by
2.5--8.7 percentage points. On a Unitree G1 humanoid and a bimanual UR3 robot, it improves success over SVDQuant by 12.8--17.6 percentage points.
\end{abstract}

\section{Introduction}
\label{introduction}

World Action Models (WAMs) couple action generation with prediction of future visual states to learn manipulation policies~\citep{lingbot-va,dreamzero,fastwam}. Despite strong task performance, their large backbones and iterative denoising make deployment costly. Recent efforts accelerate inference by removing test-time future generation~\citep{fastwam}, simplifying action modules~\citep{fasterwam}, and distilling denoising steps~\citep{flashwam,dido}. Memory demand remains a deployment bottleneck; for example, running ImageWAM~\citep{imagewam} in bf16 requires 19.9 GB of GPU memory. Post-training quantization complements these advances by reducing weight and activation precision without retraining the full model. Achieving these savings requires keeping the generated actions accurate for reliable closed-loop control.

Mixed-precision quantization balances compression and accuracy by keeping sensitive components at higher precision and quantizing the rest more aggressively. Existing VLA methods allocate precision across weight tensors or individual channels~\citep{actquant,qvla}. SVDQuant~\citep{svdquant} instead keeps a small low-rank part of each layer in 16 bits. This moves activation outliers into the weights and stores the dominant components of each weight in a 16-bit low-rank branch, so that only the residual is quantized to 4 bits. This works well for text-to-image diffusion models, but like most post-training methods, SVDQuant decides what to protect by reducing the quantization error of each linear layer's output on its own.
In a WAM, however, what matters is the generated action, and a layer's own output error says little about how much the action changes once that error passes through the rest of the network and the remaining denoising steps. Because visual world modeling and action generation are tightly coupled through iterative denoising, choosing what to keep in higher precision requires tracing how quantization errors propagate through the full action-generation process. Our analysis shows that this downstream action sensitivity concentrates in compact activation subspaces, motivating the central question: \emph{which activation directions should retain higher precision to keep the generated action accurate?}

A natural first question is whether reducing activation outliers suffices to preserve action fidelity. Smoothing~\citep{smoothquant} and rotation~\citep{quarot2024} reduce these outliers, but rounding the resulting activations to 4 bits still introduces considerable error, leaving existing quantization methods unsuitable for WAMs. Figure~\ref{fig:motivation} illustrates the distinction between activation magnitude and action sensitivity: smoothing and rotation make activation magnitudes more uniform (Figure~\subfigref{fig:motivation}{b}), yet the resulting quantization errors can still distort the final action along a few sensitive directions (Figure~\subfigref{fig:motivation}{c}). In the examined layers, a compact subspace accounts for much of the final-action sensitivity, motivating selective protection of these directions (Figure~\subfigref{fig:motivation}{d}). This suggests keeping the action-sensitive subspace in higher precision and the complementary computation in W4A4.

\begin{figure}[t]\centering
\includegraphics[width=\textwidth]{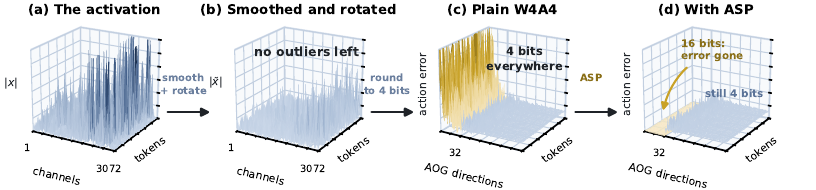}
\caption{\textbf{Removing outliers is not enough.} \textbf{(a)} The activation entering a layer has a
few outliers. \textbf{(b)} Smoothing and rotation remove them. \textbf{(c)} Quantizing
the layer still damages the action, through a few AOG directions. \textbf{(d)} Our proposed ASP protects those
directions and the damage is gone. }
\label{fig:motivation}
\end{figure}

Building on these observations, we propose \textbf{Q-WAM}, a post-training quantization framework for W4A4 world-action models. Q-WAM estimates an Action Observability Gramian (AOG), a per-layer sensitivity matrix measuring how activation errors affect the final action, using label-free randomized probes through the unrolled denoising process during calibration. Furthermore, we introduce Action-Subspace Protection (ASP), which splits each protected linear layer into two parallel computations whose outputs are summed. A low-dimensional branch projects activations onto the AOG's most action-sensitive directions and applies the corresponding projected weights in 16-bit precision, while the complementary computation uses 4-bit weights and activations. Since the selected subspace is fixed during inference, the 16-bit branch is small and adds negligible computational overhead. For a fixed protection rank, we show that the selected subspace minimizes a local approximation to action distortion when activation rounding errors have equal variance in all directions. For WAMs with Mixture-of-Transformers architecture, Q-WAM aggregates AOG sensitivity scores to concentrate protection on experts with high estimated action sensitivity.

We evaluate Q-WAM on three WAMs across 50 RoboTwin 2.0 tasks~\citep{robotwin} and five real-world tasks on bimanual UR3 arms and a Unitree G1 humanoid. In simulation, Q-WAM limits the success rate drop from bf16 to at most 1.03 points while reducing targeted-block memory by 68--71\%, and on physical robots it improves success over SVDQuant, the strongest prior method, by 12.8--17.6 points. Our contributions are as follows:

\begin{itemize}
    \item We introduce the Action Observability Gramian (AOG), a per-layer matrix predicting how much a rounding error in each weighted combination of the layer's input channels changes the final action. We estimate it label-free with random probes.

    \item We propose Action-Subspace Protection (ASP), which keeps the most action-sensitive channel combinations of each layer in a small 16-bit branch and quantizes the rest to W4A4. Moreover, we introduced a method to use the AOG to decide which experts to protect.

    \item Across Fast-WAM, ImageWAM, and LingBot-VA, Q-WAM preserves near-bf16 success on RoboTwin 2.0 with substantially reduced memory. Evaluation on Unitree G1 and bimanual UR3 demonstrates improved physical task success over SVDQuant, while ablations examine subspace protection, weight quantization, and expert selection.
    
\end{itemize}

\section{Related Works}
\label{related_works}

\paragraph{World Action Models}
World Action Models couple action generation with prediction of future visual states, using world modeling to learn representations for control. UWM~\citep{uwm} couples video and action diffusion, WorldVLA~\citep{worldvla} unifies autoregressive image and action generation, and DreamZero~\citep{dreamzero} and LingBot-VA~\citep{lingbot-va} jointly model video and actions for manipulation. Motus~\citep{motus} combines understanding, video, and action experts in a Mixture-of-Transformers, and Cosmos Policy~\citep{cosmospolicy} adds actions, future observations, and values to a pretrained video model's latent diffusion. Recent work reduces deployment cost: Fast-WAM~\citep{fastwam} and GigaWorld-Policy~\citep{gigaworldpolicy} keep video prediction only as training supervision, ImageWAM~\citep{imagewam} conditions actions on image-editing representations, and Flash-WAM~\citep{flashwam} distills sampling steps.

\paragraph{Post-Training Quantization}
PTQ reaches W4A4 in language models by smoothing~\citep{smoothquant} or rotating~\citep{quarot2024,spinquant} activation outliers, and in diffusion models by handling timestep-dependent activations~\citep{qdiffusion,ptq4dit2024,vidit-q} or absorbing outliers into a 16-bit low-rank branch~\citep{svdquant}, whose subspace is chosen from a local statistic, the weight's singular vectors. For robot policies, QVLA~\citep{qvla} and ActQuant~\citep{actquant} use action sensitivity but allocate precision to channels or whole tensors, while QuantVLA~\citep{quantvla} calibrates scales for the action head. Q-WAM instead selects the protected subspace by the sensitivity of the generated action through the full denoising process.

\paragraph{Concurrent Work}
QuantWAMs~\citep{quantwam} studies WAM quantization through shared-basis outlier calibration, joint video--action saliency for layer-wise precision allocation, and rollout-guided denoising-step protection. Q-WAM focuses on protecting dense activation subspaces selected from the sensitivity of the final generated action through the unrolled sampler.

\section{Preliminaries}
\label{preliminaries}

\paragraph{World Action Models}
\label{sec:prelim-wam}
At each control step $t$, a policy receives an observation $o_t$, made of camera views and, for
some policies, the robot's proprioceptive state, together with a language instruction $l$. It
outputs an action chunk $a=a_{t:t+h}\in\mathbb{R}^{m}$ of $h$ future steps of a $d_a$-dimensional
control, with $m=h\,d_a$. A World Action Model produces this chunk through a world model. The
observation and the instruction are encoded once into a cached conditioning $\z(o_t,l)$, which an
action head reads at every denoising step. Some WAMs implement the world model and
the action head as a Mixture-of-Transformers~\citep{mot2024}, in which per-modality experts hold
separate weights and interact only through joint attention. Others share a single backbone between
video and action~\citep{lingbot-va}. The action head is trained by flow
matching~\citep{flowmatching2023}. At inference it starts from $a^{(0)}\sim\mathcal{N}(0,I)$ and
integrates the learned velocity field $v_\theta$ over $T$ Euler steps,
\begin{equation}
\label{eq:denoise}
a^{(s+1)}=a^{(s)}+\eta_s\,v_\theta\!\left(a^{(s)},\tau_s,\z\right),\qquad s=0,\dots,T-1,
\end{equation}
and returns $a=a^{(T)}$. Each step is one transformer pass over the $h$ action tokens, one per future step.

\paragraph{Post-Training Quantization}
\label{sec:prelim-ptq}
A linear layer $\ell$ maps an activation $x_\ell\in\mathbb{R}^{d_\ell}$ to $y_\ell=W_\ell^\top x_\ell$
with $W_\ell\in\mathbb{R}^{d_\ell\times d'_\ell}$. Given a tensor $X$, $b$-bit
quantization maps $X$ onto a scaled integer grid,
\begin{equation}
\label{eq:quant}
Q_b(X)=\varsigma\cdot\operatorname{round}\!\left(X/\varsigma\right),\qquad
\varsigma=\max|X|\,/\,(2^{b-1}-1),
\end{equation}
with one scale $\varsigma$ per group of consecutive entries. Weights are quantized once offline and
activations per token at run time, and with 4-bit weights and activations the layer computes
$Q_4(W_\ell)^\top Q_4(x_\ell)$ as one integer general matrix multiplication (GEMM). The largest entry
of a group sets $\varsigma$, so a single outlier coarsens the grid for the whole group, and
activation outliers reach one to two orders of magnitude above typical values and move with the
input. We therefore apply two exact transforms before quantization.
SmoothQuant~\citep{smoothquant} migrates outlier magnitude from the activation to the weight with a
per-channel factor $\gamma\in\mathbb{R}^{d_\ell}$, and a block Hadamard rotation
$H$~\citep{quarot2024} spreads the remaining outliers across coordinates,
\begin{equation}
\label{eq:pair}
\xt=H\diag(\gamma)^{-1}x,\qquad \Wt=H\diag(\gamma)\,W,\qquad \Wt^\top\xt=W^\top x.
\end{equation}
A WAM block generates its normalization parameters at run time, so we apply $H$ online with the fast
Walsh--Hadamard transform~\citep{fino1976}. The deployed layer computes $Q_4(\Wt)^\top Q_4(\xt)$, and
Section~\ref{sec:method} works in these transformed coordinates.

\section{Method}
\label{method}
\label{sec:method}
We propose Q-WAM, a post-training W4A4 quantization method for World Action Models that allocates
activation precision by the sensitivity of the generated action. Section~\ref{sec:formulation}
defines and estimates this sensitivity, the Action Observability Gramian (AOG).
Section~\ref{sec:asp} introduces Action-Subspace Protection (ASP), which keeps the activation
directions that the AOG marks as most sensitive in 16 bits and quantizes the remaining weights and
activations to 4 bits. Section~\ref{sec:where} aggregates the AOG over the layers of each expert to
select the experts that receive this protection.

\begin{figure}[t]\centering
\includegraphics[width=\textwidth]{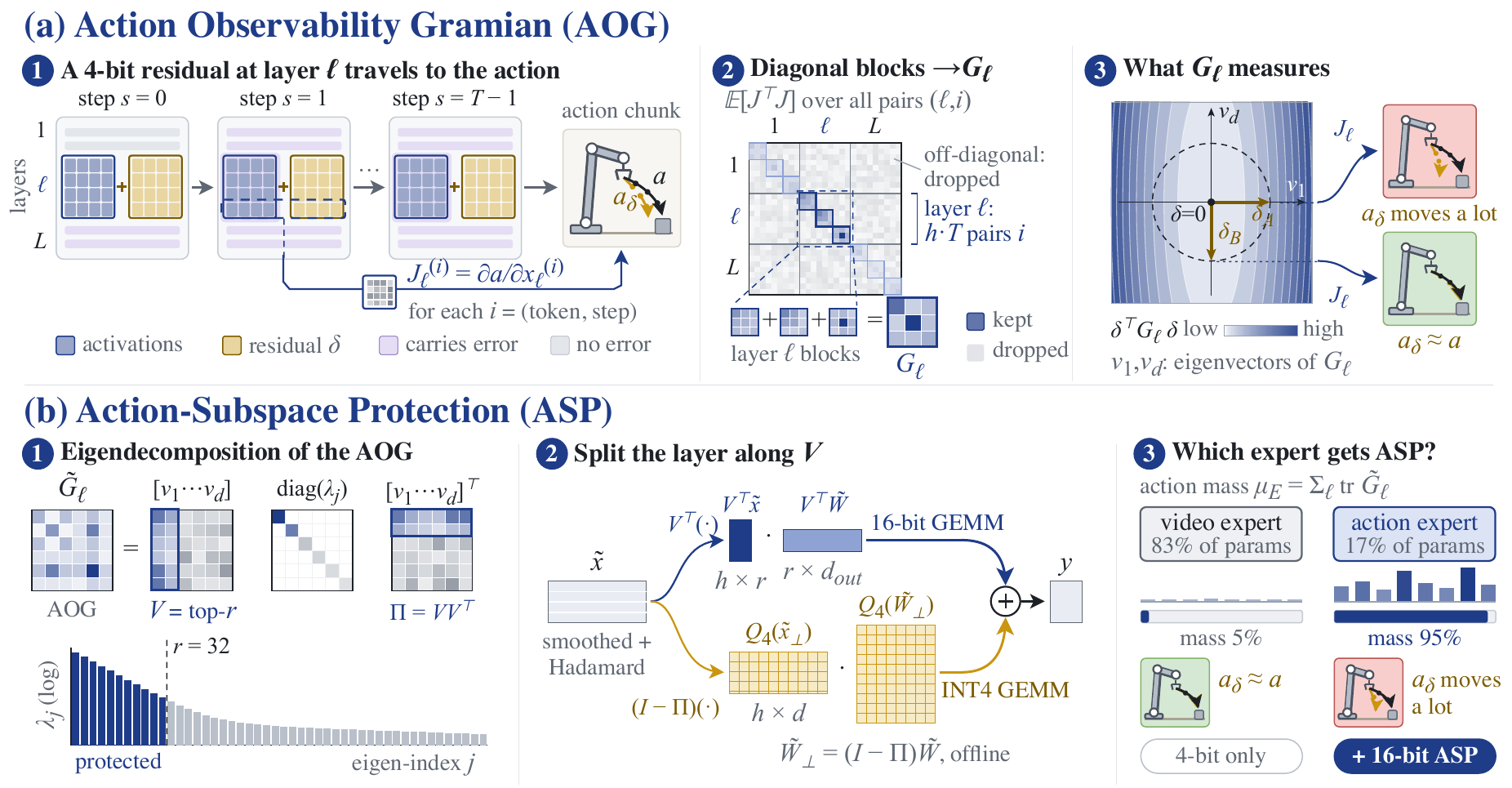}
\caption{\textbf{Overview of Q-WAM.} \textbf{(a)} The AOG $G_\ell$ measures how far a rounding error in each weighted combination of layer $\ell$'s input channels moves the action in later denoising steps. \textbf{(b)} ASP keeps the most sensitive of these combinations in a 16-bit branch and quantizes the deflated remainder to 4 bits; the action mass decides which experts receive it. }
\label{fig:overview}
\end{figure}

\subsection{The Action Observability Gramian}
\label{sec:formulation}
Within one action prediction, layer $\ell$ runs once for each of the $h$ action tokens at each of the
$T$ steps of Equation~\ref{eq:denoise}. We index these token--step pairs by $i\in\mathcal{I}_\ell$,
write $x^{(i)}_\ell$ for the input activation at pair $i$, and write
$\delta^{(i)}_\ell=Q_4(x^{(i)}_\ell)-x^{(i)}_\ell$ for its quantization residual. The residual changes
with the input, so we treat it as random vector over calibration inputs $(o_t,l)$. Let $\actionfun$ be the action chunk $a$ generated by the full-precision
model and $\actionfun_\delta$ the chunk generated when the activations carry the residuals
$\delta=(\delta^{(i)}_\ell)_{\ell,i}$. A first-order Taylor expansion gives
\begin{equation}
\label{eq:taylor}
\actionfun_\delta-\actionfun
=\sum_{\ell}\sum_{i\in\mathcal{I}_\ell}J^{(i)}_\ell\,\delta^{(i)}_\ell+O(\|\delta\|^2),
\qquad
J^{(i)}_\ell=\frac{\partial a}{\partial x^{(i)}_\ell}\in\mathbb{R}^{m\times d_\ell},
\end{equation}
The Jacobian $J^{(i)}_\ell$ says how the final action responds to a small change in the input of
layer $\ell$ at pair $i$. It accounts for everything that happens after that point: the rest of the
block, and all remaining denoising steps. Our goal is that the
quantized model generates the same action as the full-precision model, so we minimize the expected
squared change of the action caused by the rounding errors. Inserting Equation~\ref{eq:taylor} into
this squared change gives one term for every layer and token--step pair, plus cross terms between
different pairs. We treat the rounding errors as zero-mean, uncorrelated with each other, and
uncorrelated with the Jacobians (Appendix~\ref{app:blockdiag} states these approximations
precisely). The cross terms then vanish, and since all token--step pairs of a layer share one error
distribution, what remains is a sum with one term per layer:
\begin{equation}
\label{eq:obj}
\mathbb{E}\big\|\actionfun_\delta-\actionfun\big\|^2\;\approx\;\sum_{\ell}\mathbb{E}\big[\delta_\ell^\top G_\ell\,\delta_\ell\big],\qquad
G_\ell=\mathbb{E}\Big[\sum_{i\in\mathcal{I}_\ell}J^{(i)\top}_\ell J^{(i)}_\ell\Big]
\in\mathbb{R}^{d_\ell\times d_\ell}.
\end{equation}
We call $G_\ell$ the Action Observability Gramian (AOG) of layer $\ell$. Equation~\ref{eq:obj} says
what the AOG measures: if the input of layer $\ell$ carries a rounding error $\delta$, then
$\delta^\top G_\ell\,\delta$ approximates the expected squared change this error causes in the final
action. This
is useful for two reasons. First, it turns a question about the whole model, how far quantization
moves the robot's action, into one matrix per layer that we compute once during calibration.
Second, because $G_\ell$ is a matrix and not a single score, it tells how sensitive the action is to
errors in every direction of the layer's input, that is, in every weighted combination of its
channels. For each
layer, this quadratic form is also exact to second order: at the full-precision model, $G_\ell$ is the
Hessian of $\frac12\mathbb{E}\|\actionfun_\delta-\actionfun\|^2$ with respect to the layer's input
(Appendix~\ref{app:gn}).

\paragraph{Estimating the AOG $G_\ell$.}
Computing $G_\ell$ exactly is expensive: $J^{(i)}_\ell$ has one row per action coordinate, $hd_a$ in
total, and each row needs its own backward pass through all denoising steps. We avoid this with random probes $u\sim\mathcal{N}(0,I_m)$. A single
backward pass of $\langle a,u\rangle$ returns $J^{(i)\top}_\ell u$ for every layer and token--step pair
at once~\citep{baydin2018}, and since $\mathbb{E}[uu^\top]=I_m$, the expectation of
$(J^\top u)(J^\top u)^\top$ over probes equals $J^\top J$ for every Jacobian $J$. Averaging over $N$ calibration inputs with $P$
probes each therefore gives the unbiased estimate
\begin{equation}
\label{eq:ghat}
\widehat{G}_\ell=\frac{1}{NP}\sum_{n=1}^{N}\sum_{p=1}^{P}\sum_{i\in\mathcal{I}_\ell}
\big(J^{(i)\top}_{\ell,n}u_{np}\big)\big(J^{(i)\top}_{\ell,n}u_{np}\big)^\top,
\end{equation}
where $J^{(i)}_{\ell,n}$ is the Jacobian at input $n$. Each input then costs $P$ backward passes
instead of $hd_a$; we use $P{=}12$, about $37\times$ fewer than the exact computation on Fast-WAM. Appendix~\ref{app:probe} gives the procedure and bounds the error of the
trace and of the leading subspace of $\widehat{G}_\ell$.

\paragraph{The AOG predicts quantization damage.}
We check this prediction on Fast-WAM. For each of the linear layers in its action expert, we
quantize that layer alone to W4A4 and measure the normalized root-mean-square error (NRMSE) of the
generated action, which is the damage from quantizing the layer. Figure~\subfigref{fig:method}{a} compares
this damage with two per-layer criteria. Each point is one layer, placed by its damage and by its
score under a criterion, and each score is divided by its median over the layers so that both
criteria share one axis. The first is the local error $\|\delta_\ell\|^2$, the
squared quantization error of the layer input. Round-to-nearest minimizes it by construction, and
widely used post-training methods~\citep{awq,gptq,svdquant} calibrate
each layer or block from local quantities of this kind, the error or statistics of its own input and
output, without looking at the generated action. The second is the AOG prediction
$\mathbb{E}[\delta_\ell^\top G_\ell\delta_\ell]$. The AOG prediction tracks the damage with a
Pearson correlation $\rho{=}0.95$, while the local error reaches $\rho{=}0.04$.

\begin{figure}[t]\centering
\includegraphics[width=\textwidth]{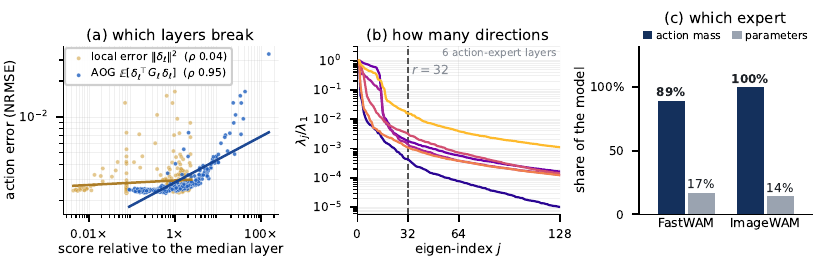}
\setlength{\abovecaptionskip}{2pt}%
\caption{\textbf{Deciding what to protect with the AOG.} \textbf{(a)} Action error from quantizing a
single layer, against the layer's local error and its AOG score.
\textbf{(b)} The first 128 eigenvalues of the rotated AOG for several action-expert layers, which fall
quickly; the dashed line marks the rank we keep. \textbf{(c)} The action expert's share of the
action mass and of the parameters in each model.}
\label{fig:method}
\end{figure}

\subsection{Action-Subspace Protection}
\label{sec:asp}
The AOG weighs each direction of the activation by its effect on the action. A residual
$\delta_\ell$ costs $\mathbb{E}[\delta_\ell^\top G_\ell\,\delta_\ell]$, and we keep the few most costly
directions in 16 bits. We find these directions in three steps. First, we apply the smoothing and
rotation of Section~\ref{sec:prelim-ptq}, so the layer quantizes $\xt=H\diag(\gamma)^{-1}x$ instead
of $x$. Second, we express the AOG in these coordinates. Since $x=\diag(\gamma)H\xt$, a rounding
error $\delta$ of $\xt$ shifts $x$ by $\diag(\gamma)H\delta$ and therefore costs
$\delta^\top\Gt_\ell\,\delta$, with $\Gt_\ell=H\diag(\gamma)\,G_\ell\,\diag(\gamma)H$. Third, since
$\Gt_\ell$ is symmetric and positive semidefinite, it has orthonormal eigenvectors $v_j$,
$\Gt_\ell=\sum_j\lambda_j v_jv_j^\top$, and
\begin{equation}
\label{eq:gtilde}
\delta^\top\Gt_\ell\,\delta=\sum_j\lambda_j\big(v_j^\top\delta\big)^2,\qquad \lambda_1\ge\lambda_2\ge\dots\ge0 .
\end{equation}
The cost splits into one term per eigenvector: the component $v_j^\top\delta$ of the residual contributes
$\lambda_j(v_j^\top\delta)^2$. A large eigenvalue marks a direction in which rounding changes the
action strongly, and a small eigenvalue marks a direction to which the action barely reacts. The
eigenvalues fall by orders of magnitude within a few dozen directions
(Figure~\subfigref{fig:method}{b}), so only a few directions of each activation are sensitive.
Each direction mixes many channels and is set by the action's response to errors, not by
activation values, so we protect directions rather than channels. We call the span of the top $r$ eigenvectors the layer's
action subspace, keep the component of $\xt$ inside it in 16 bits, and quantize the rest to 4 bits.

\paragraph{Splitting the layer.}
Let $V\in\mathbb{R}^{d_\ell\times r}$ hold the top-$r$ eigenvectors of $\Gt_\ell$, so that
$V^\top V=I_r$, and let $\Pi=VV^\top$ project onto the action subspace. We split every activation into
a protected component inside the action subspace and a remainder orthogonal to it,
\begin{equation}
\label{eq:xsplit}
\xt=\Pi\xt+(I-\Pi)\xt
=\underbrace{V\big(V^\top\xt\big)}_{\text{protected component}}
+\underbrace{\xt_\perp}_{\text{remainder}},\qquad \xt_\perp=(I-\Pi)\xt.
\end{equation}
Since $W^\top x=\Wt^\top\xt$, multiplying Equation~\ref{eq:xsplit} by $\Wt^\top$ splits the output of
the layer into two products. The first, $\Wt^\top V(V^\top\xt)=(V^\top\Wt)^\top(V^\top\xt)$, runs
over only $r\ll d_\ell$ dimensions. In the second, $I-\Pi$ is a symmetric projection, so it leaves
$\xt_\perp$ unchanged and can be moved onto the weight,
$\Wt^\top\xt_\perp=\Wt^\top(I-\Pi)\xt_\perp=\big((I-\Pi)\Wt\big)^\top\xt_\perp=\Wt_\perp^\top\xt_\perp$,
where $\Wt_\perp=(I-\Pi)\Wt$ is the weight with its protected component removed. We keep the first
product in 16 bits and quantize both factors of the second to 4 bits,
\begin{equation}
\label{eq:asp}
W^\top x=\Wt^\top\xt
=\big(V^\top\Wt\big)^{\top}\big(V^\top\xt\big)+\Wt_\perp^{\top}\xt_\perp
\approx\underbrace{\big(V^\top\Wt\big)^{\top}\big(V^\top\xt\big)}_{\text{low-rank 16-bit branch}}
+\underbrace{Q_4(\Wt_\perp)^{\top}Q_4(\xt_\perp)}_{\text{4-bit deflated path}}.
\end{equation}
We call this construction Action-Subspace Protection (ASP). Projecting $\xt$ before $Q_4$ keeps the
protected component out of the 4-bit grid, and deflating the weight to $\Wt_\perp$ lets the rounding
error $\delta$ reach the output only as $(I-\Pi)\delta$, up to the product of the weight and
activation rounding errors. Both paths are dense GEMMs that run efficiently on GPUs, and the thin rank-$r$ branch adds
little overhead (Appendix~\ref{app:overhead}).

By Equation~\ref{eq:gtilde}, the damage along $v_j$ depends on how strongly the action
reacts to that direction ($\lambda_j$) and how much rounding error lands on it. After the
rotation, the rounding error has roughly the same mean square $\sigma^2$ in every direction, i.e., it
is \emph{isotropic}, so the expected damage along $v_j$ is $\sigma^2\lambda_j$. The top-$r$ eigenvectors
are therefore the best $r$ directions to protect, which Appendix~\ref{app:kyfan} proves.

\subsection{Where to Protect: The action mass of an expert}
\label{sec:where}

Each protected layer carries a 16-bit branch, so in a Mixture-of-Transformers we protect only the
experts whose rounding errors damage the action most. We denote an expert by $E$, the set of layers
that process one modality, and score it by the expected damage its 4-bit rounding does to the
action. With the isotropic rounding error of Section~\ref{sec:asp}, every direction of a layer
receives error $\sigma^2$, so the layer's term in Equation~\ref{eq:obj} is $\sigma^2$ times the sum
of its eigenvalues, $\mathbb{E}[\delta_\ell^\top\Gt_\ell\delta_\ell]=\sigma^2\tr\Gt_\ell$. Taking
the same $\sigma^2$ for all layers and summing over the layers of $E$, the expert contributes
$\sigma^2\mu_E$, where
\begin{equation}
\label{eq:mass}
\mu_E=\sum_{\ell\in E}\tr\big(\Gt_\ell\big)
\end{equation}
is the action mass of $E$. Protecting a layer removes the damage along its top-$r$ eigenvectors,
which is the most that any rank-$r$ branch can remove; Appendix~\ref{app:kyfan} proves this. As
depicted in Figure~\subfigref{fig:method}{b}, the eigenvalues fall steeply, so the protected
directions carry nearly all of a layer's damage, and the gain from protecting an expert grows with
its action mass. ASP therefore protects the experts with the largest action mass per added 16-bit
value. As depicted in Figure~\subfigref{fig:method}{c}, on both Mixture-of-Transformers models the
action expert holds under a fifth of the parameters yet carries $89.3\%$ of the action mass on
Fast-WAM (Appendix~\ref{app:aogmap} visualize it per layer) and $99.99\%$ on ImageWAM, so ASP protects
the action expert alone. LingBot-VA shares one backbone, so ASP covers
all of its layers. Furthermore, table~\ref{tab:whereasp} shows that protecting the other expert
changes the success rate by less than the run-to-run RoboTwin variation.

\section{Experiments}
\label{sec:exp}

\begin{table}[t]\centering\small
\renewcommand{\arraystretch}{0.92}%
\caption{RoboTwin 2.0 results. BPW is the weight cost; Mem is the targeted-block memory.}
\label{tab:main}
\newcommand{\mainrows}{%
\toprule
Model & Method & Precision & BPW & Clean & Randomized & Avg. & Mem (GB) \\
\midrule
\multicolumn{8}{l}{\emph{Fast-WAM ($2$-expert MoT, video generation backbone)}}\\
 & bf16 (upper bound) & W16A16 & $16.00$ & $92.28$ & $91.34$ & $91.81$ & $11.85$ \\
 & SmoothQuant        & W4A4   & $4.00$  & $21.98$ & $16.48$ & $19.23$ & $2.98$ \\
 & ViDiT-Q            & W4A4   & $4.41$  & $82.00$ & $80.48$ & $81.24$ & $3.61$ \\
 & SVDQuant           & W4A4   & $4.84$  & $89.12$ & $87.38$ & $88.25$ & $3.60$ \\
\rowcolor{oursrow} & \textbf{Q-WAM (Ours)} & W4A4 & $\mathbf{4.62}$ & $\mathbf{91.62}$ & $\mathbf{89.94}$ & $\mathbf{90.78}$ & $3.44$ \\
\midrule
\multicolumn{8}{l}{\emph{ImageWAM ($2$-expert MoT, image editing backbone)}}\\
 & bf16 (upper bound) & W16A16 & $16.00$ & $92.82$ & $93.70$ & $93.26$ & $9.11$ \\
 & SmoothQuant        & W4A4   & $4.00$  & $59.04$ & $59.36$ & $59.20$ & $2.36$ \\
 & ViDiT-Q            & W4A4   & $4.76$  & $78.76$ & $82.08$ & $80.42$ & $3.05$ \\
 & SVDQuant           & W4A4   & $4.75$  & $84.34$ & $84.22$ & $84.28$ & $2.78$ \\
\rowcolor{oursrow} & \textbf{Q-WAM (Ours)} & W4A4 & $\mathbf{4.58}$ & $\mathbf{93.00}$ & $\mathbf{92.94}$ & $\mathbf{92.97}$ & $2.69$ \\
\midrule
\multicolumn{8}{l}{\emph{LingBot-VA (shared backbone, not an MoT)}}\\
 & bf16 (upper bound) & W16A16 & $16.00$ & $90.84$ & $90.20$ & $90.52$ & $10.16$ \\
 & SmoothQuant        & W4A4   & $4.00$  & $81.68$ & $72.24$ & $76.96$ & $2.80$ \\
 & ViDiT-Q            & W4A4   & $4.41$  & $80.82$ & $74.04$ & $77.43$ & $3.18$ \\
 & SVDQuant           & W4A4   & $4.76$  & $85.84$ & $83.12$ & $84.48$ & $3.27$ \\
\rowcolor{oursrow} & \textbf{Q-WAM (Ours)} & W4A4 & $\mathbf{4.76}$ & $\mathbf{90.20}$ & $\mathbf{88.92}$ & $\mathbf{89.56}$ & $3.27$ \\
\bottomrule}
\setlength{\tabcolsep}{0pt}%
\sbox0{\begin{tabular}{llcccccc}\mainrows\end{tabular}}%
\setlength{\tabcolsep}{\dimexpr(0.95\textwidth-\wd0)/16\relax}%
\begin{tabular}{llcccccc}\mainrows\end{tabular}
\end{table}

\paragraph{Experimental Setup} We evaluate Q-WAM on three WAM designs: Fast-WAM~\citep{fastwam} and ImageWAM~\citep{imagewam} add an action expert to a video-generation and an image-editing backbone in a Mixture-of-Transformers,
and the released open-sourced LingBot-VA~\citep{lingbot-va} uses one shared backbone for video and action. We use a protection rank of $r=32$ and group-wise W4A4
quantization with a group size of 32, retaining the protected branches in 16-bit precision. Calibration runs on NVIDIA H100 GPUs, simulation evaluation on NVIDIA L40S GPUs, and real-world inference on an NVIDIA RTX 5090. Implementation details are provided in Appendix~\ref{app1}. 

\paragraph{Benchmarks and Baselines}
We evaluate all three models on RoboTwin 2.0~\citep{robotwin} simulation under clean and randomized conditions, with 100 demonstrations. Real-world evaluation covers Fast-WAM and ImageWAM on five tasks across two embodiments (Figure~\ref{fig:realworld}). The Unitree G1 performs \emph{tool sorting} and \emph{item classification}, and a bimanual UR3 arm performs \emph{table cleaning}, \emph{cube stacking}, and \emph{drawer manipulation}, with 25 trials per task. We report task success rates and compare against bf16 policies, SmoothQuant~\citep{smoothquant}, ViDiT-Q~\citep{vidit-q}, and SVDQuant~\citep{svdquant} in simulation, and bf16 and SVDQuant on physical robots. Quantized methods use W4A4, with effective bits per weight (BPW) accounting for storage overhead.

\subsection{Main Results}

\begin{figure}[t]
    \centering
    \includegraphics[width=\linewidth]{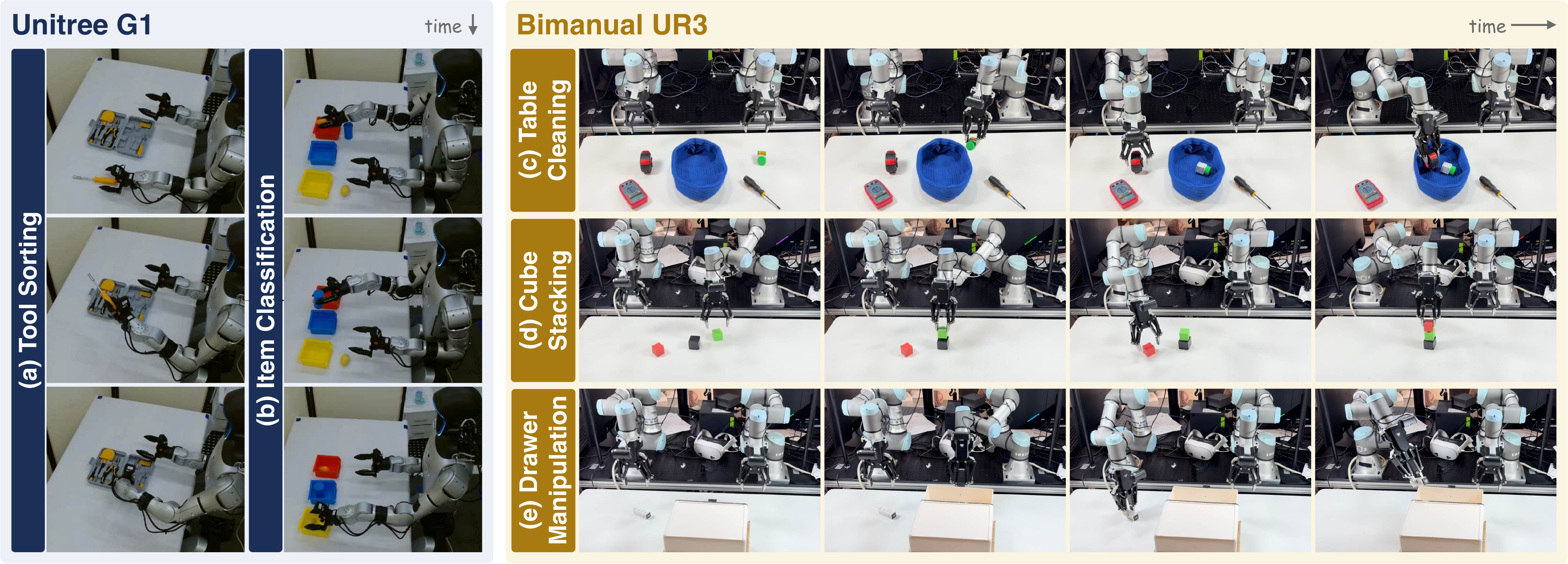}
    \setlength{\abovecaptionskip}{4pt}%
    \caption{Real-world evaluation tasks. Frames progress from left to right within each task.}
    \label{fig:realworld}
\end{figure}

\begin{table}[t]\centering\small
\caption{Real-world results (successes/trials) on two embodiments and five tasks.}
\label{tab:realworld}
\newcommand{\realworldrows}{%
\toprule
& & \multicolumn{2}{c}{Unitree G1} & \multicolumn{3}{c}{UR3 Bimanual} & \\
\cmidrule(lr){3-4} \cmidrule(lr){5-7}
Model & Method & Tool Sort & Item Cls. & Table\ Clean & Drawer & Cube Stacking & Avg. (\%) \\
\midrule
& bf16     & 19/25 & 13/25 & 17/25 & 20/25 & 12/25 & 64.8 \\
& SVDQuant & 16/25 & 10/25 & 12/25 & 12/25 & 4/25  & 43.2 \\
\rowcolor{oursrow}
\multirow{-3}{*}{\emph{Fast-WAM}}
& \textbf{Q-WAM (Ours)} & \textbf{18/25} & \textbf{11/25} & \textbf{14/25} & \textbf{19/25} & \textbf{8/25} & \textbf{56.0} \\
\midrule
& bf16     & 21/25 & 17/25 & 22/25 & 23/25 & 21/25 & 83.2 \\
& SVDQuant & 15/25 & 12/25 & 19/25 & \textbf{22/25} & 5/25 & 58.4 \\
\rowcolor{oursrow}
\multirow{-3}{*}{\emph{ImageWAM}}
& \textbf{Q-WAM (Ours)} & \textbf{19/25} & \textbf{14/25} & \textbf{21/25} & \textbf{22/25} & \textbf{19/25} & \textbf{76.0} \\
\bottomrule}
\setlength{\tabcolsep}{0pt}%
\sbox0{\begin{tabular}{llcccccc}\realworldrows\end{tabular}}%
\setlength{\tabcolsep}{\dimexpr(0.9\textwidth-\wd0)/16\relax}%
\begin{tabular}{llcccccc}\realworldrows\end{tabular}
\end{table}

\paragraph{Simulation Results}
Table~\ref{tab:main} shows that Q-WAM achieves the highest
success rates among the quantized methods in both clean
and randomized conditions across all three models.
Its average success reaches 90.78\% on Fast-WAM,
92.97\% on ImageWAM, and 89.56\% on LingBot-VA,
trailing the respective bf16 baselines by only 1.03,
0.29, and 0.96 percentage points.
Compared with SVDQuant, Q-WAM improves average success
by 2.53, 8.69, and 5.08 percentage points, respectively.
These improvements come at lower storage cost on the
two multi-expert models. BPW decreases from 4.84 to
4.62 on Fast-WAM and from 4.75 to 4.58 on ImageWAM.
On LingBot-VA, Q-WAM achieves its gain at the same
4.76 BPW and 3.27\,GB of memory as SVDQuant.
The improvements persist under randomized conditions,
where Q-WAM exceeds SVDQuant by 2.56, 8.72, and
5.80 points across the three models.
In comparison, SmoothQuant and ViDiT-Q exhibit
substantial model-dependent degradation, particularly
on ImageWAM, where their average success rates fall
to 59.20\% and 80.42\%, versus 92.97\% for Q-WAM.
Overall, Q-WAM preserves near-bf16 task performance
across both multi-expert and shared-backbone WAMs
while reducing the memory of the targeted blocks
by 68--71\% (Table~\ref{tab:main}). Figure~\ref{fig:latency} reports latency with real
4-bit kernels on one NVIDIA L40S: the quantized layers that process the action tokens run
$1.06\times$, $1.56\times$ and $1.41\times$ faster than bf16 on Fast-WAM, ImageWAM and LingBot-VA.
These layers handle short action chunks, so their cost is dominated by loading weights, which
4-bit storage reduces.

\paragraph{Real-World Experiments}
Table~\ref{tab:realworld} evaluates whether these gains
extend to physical execution across two embodiments.
Q-WAM achieves average success rates of 56.0\% on Fast-WAM
and 76.0\% on ImageWAM, exceeding SVDQuant by 12.8 and
17.6 percentage points, respectively. It matches or exceeds
SVDQuant on every model--task pair. The largest improvement
occurs on ImageWAM cube stacking, where Q-WAM succeeds in
19/25 trials, compared with 5/25 for SVDQuant and 21/25
for bf16. While gaps of 8.8 and 7.2 points to bf16 remain,
the results demonstrate improved closed-loop performance
under quantization on physical robots.

\begin{table}[t]\centering\small
\caption{Component ablation on RoboTwin 2.0 at a fixed weight group size.  }
\label{tab:ablation}
\setlength{\tabcolsep}{4pt}
\begin{tabular}{llccccc}
\toprule
Model & Method & BPW & Clean & Randomized & Avg. & Mem (GB) \\
\midrule
& bf16 (upper bound)                  & 16.00 & 92.28 & 91.34 & 91.81 & 11.85 \\
& per-group W4A4                      & 4.50  & 82.60 & 81.80 & 82.20 & 3.35 \\
& $+$ smoothing and rotation      & 4.50  & 87.78 & 87.32 & 87.55 & 3.35 \\
\rowcolor{oursrow}
\multirow{-4}{*}{\emph{Fast-WAM}}
& \textbf{$+$ ASP (Ours)}             & \textbf{4.62} & \textbf{91.62} & \textbf{89.94} & \textbf{90.78} & 3.44 \\
\midrule
& bf16 (upper bound)                  & 16.00 & 92.82 & 93.70 & 93.26 & 9.11 \\
& per-group W4A4                      & 4.50  & 21.16 & 20.14 & 20.65 & 2.64 \\
& $+$ smoothing and rotation      & 4.50  & 86.78 & 88.06 & 87.42 & 2.64 \\
\rowcolor{oursrow}
\multirow{-4}{*}{\emph{ImageWAM}}
& \textbf{$+$ ASP (Ours)}             & \textbf{4.58} & \textbf{93.00} & \textbf{92.94} & \textbf{92.97} & 2.69 \\
\midrule
& bf16 (upper bound)                  & 16.00 & 90.84 & 90.20 & 90.52 & 10.16 \\
& per-group W4A4                      & 4.50  & 83.44 & 78.24 & 80.84 & 3.11 \\
& $+$ smoothing and rotation      & 4.50  & 87.52 & 82.64 & 85.08 & 3.11 \\
\rowcolor{oursrow}
\multirow{-4}{*}{\emph{LingBot-VA}}
& \textbf{$+$ ASP (Ours)}             & \textbf{4.76} & \textbf{90.20} & \textbf{88.92} & \textbf{89.56} & 3.27 \\
\bottomrule
\end{tabular}
\end{table}

\begin{table}[t]
\begin{minipage}[t]{0.62\textwidth}\vspace{0pt}\centering\small
\caption{Adding ASP to the second expert on top of the action expert. $^\dagger$\,The expert Q-WAM
selects by action mass.}
\label{tab:whereasp}
\setlength{\tabcolsep}{2pt}
\begin{tabular}{llcccc}
\toprule
model & ASP applied to & BPW & clean & randomized & Avg. \\
\midrule
\rowcolor{oursrow}
Fast-WAM  & action expert$^\dagger$ & 4.62 & 91.62 & 89.94 & 90.78 \\
         & \quad$+$ video expert & 4.80 & 91.04 & 90.68 & 90.86 \\
\midrule
\rowcolor{oursrow}
ImageWAM & action expert$^\dagger$ & 4.58 & 93.00 & 92.94 & 92.97 \\
         & \quad$+$ editing expert         & 4.75 & 93.52 & 93.46 & 93.49 \\
\bottomrule
\end{tabular}
\end{minipage}\hfill
\begin{minipage}[t]{0.36\textwidth}\vspace{0pt}\centering
\expandafter\def\csname @captype\endcsname{figure}
\includegraphics[width=\linewidth]{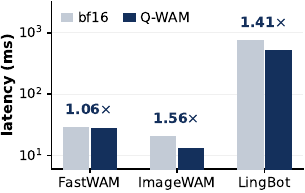}
\setlength{\abovecaptionskip}{2pt}%
\caption{\raggedright Block-level latency.}
\label{fig:latency}
\end{minipage}
\end{table}

\subsection{Ablation Analysis}
 Table~\ref{tab:ablation} adds smoothing and rotation, then ASP, to per-group W4A4 at a fixed group size. Smoothing and rotation raise average success by 5.35, 66.77, and 4.24 points on Fast-WAM, ImageWAM, and LingBot-VA without changing BPW, but still leave gaps of 4.26, 5.84, and 5.44 points to bf16. ASP recovers 3.23, 5.55, and 4.48 of these points (76--95\% of the remaining gap) under both clean and randomized conditions, for only 0.08--0.26 additional BPW (0.05--0.16\,GB). Removing outliers is therefore not enough; protecting the action-sensitive subspace is what brings W4A4 close to bf16. Appendix~\ref{app:weightside} shows that adding SVDQuant's low-rank weight protection on top of ASP costs extra bits without improving success. Moreover, Appendix~\ref{app:calib} provides a further ablation on the sensitivity to the calibration data.

 On the two multi-expert models, the results above apply ASP only to the action expert, which the action mass $\mu_E$ selects (Section~\ref{sec:where}). Table~\ref{tab:whereasp} tests whether the other expert (video or image editing) also needs it. Extending ASP to it adds 0.17--0.18 BPW, more than ASP on the action expert itself costs, yet changes average success by only $+0.08$ and $+0.52$, less than the $1.12$-point gap between two calibration sets (Appendix~\ref{app:calib}). The action expert alone therefore captures the benefit of ASP, and $\mu_E$ correctly identifies where protection is needed.

\section{Conclusion}
\label{conclusion}

In this work, we present Q-WAM, a post-training W4A4 quantization method for WAMs. To address the
gap between a layer's local error and its effect on the action, we propose the Action Observability
Gramian, which measures how much a rounding error in each weighted combination of a layer's input
channels moves the action. To address the action's sensitivity to a few of these combinations, we
propose Action-Subspace Protection, which keeps them in a small 16-bit branch and quantizes the rest
to 4 bits. On three WAMs, Q-WAM stays within $1.1$ points of bf16 on RoboTwin 2.0, reduces the
memory of the quantized blocks by $3.1$--$3.4\times$, and outperforms SVDQuant in simulation and on
two physical robots.

\subsection*{AI use statement}
In this work, we used large language models (LLMs) to implement some parts of the experiments.
The methodology and the experiments were not designed with generative AI tools, and the structure
of the paper was not generated by AI. Additionally, we used LLMs for grammar checking and polishing
the writing of the paper, and to find relevant papers for the literature review. We have reviewed
all AI-assisted work: two of the authors checked the LLM-assisted implementation, and the
literature review was also checked manually. To minimize the risk of factual inaccuracies or
citation errors, every model-edited sentence underwent human review, and all references were
carefully cross-checked with their primary sources. We take responsibility for the final content
of this work, including text, claims or artifacts produced with the aid of generative AI.

\bibliography{iclr2027_conference}
\bibliographystyle{iclr2027_conference}

\appendix

\section{Derivations for the Action Observability Gramian}
\label{app:metric}
Throughout, $\actionfun_\ell$ is the map from layer $\ell$'s input activation
$x_\ell\in\mathbb{R}^{d_\ell}$ to the generated action chunk $a\in\mathbb{R}^{m}$ through the
unrolled $T$-step denoiser of Equation~\ref{eq:denoise}, $J_\ell=\partial a/\partial x_\ell$ is
its Jacobian, and $G_\ell=\mathbb{E}[J_\ell^\top J_\ell]$ as in Equation~\ref{eq:obj}. The
expectation is over calibration inputs, as in Section~\ref{sec:formulation}, and the
token--step index $i$ is suppressed wherever it plays no role.

\subsection{$G_\ell$ is the Hessian of the action error at the full-precision point}
\label{app:gn}
For a fixed perturbation $\delta_\ell$ of layer $\ell$'s activation, define the action-error
objective
\begin{equation*}
F_\ell(\delta_\ell)=\tfrac12\,\mathbb{E}\big\|\xi(\delta_\ell)\big\|^2,\qquad
\xi(\delta_\ell)=\actionfun_\ell(x_\ell+\delta_\ell)-\actionfun_\ell(x_\ell)\in\mathbb{R}^{m},
\end{equation*}
so that $\xi(0)=0$. Differentiating twice with respect to $\delta_\ell$,
\begin{equation*}
\nabla^2F_\ell(\delta_\ell)=\mathbb{E}\Big[
\underbrace{J_\ell(\delta_\ell)^\top J_\ell(\delta_\ell)}_{\text{Gauss--Newton}}
+\underbrace{\textstyle\sum_{k=1}^{m} \xi_k(\delta_\ell)\,\nabla^2\actionfun_{\ell,k}(x_\ell+\delta_\ell)}_{\text{residual curvature}}
\Big],
\end{equation*}
where $J_\ell(\delta_\ell)$ is the Jacobian evaluated at $x_\ell+\delta_\ell$. The second term
is weighted by the components of $\xi$, which vanish at $\delta_\ell=0$, so
\begin{equation*}
\nabla^2F_\ell(0)=\mathbb{E}\big[J_\ell^\top J_\ell\big]=G_\ell.
\end{equation*}
Section~\ref{sec:formulation} uses two consequences. First, $G_\ell$ is the exact Hessian of the
action error at the full-precision point, so the quadratic form of Equation~\ref{eq:obj} needs
no assumption about training convergence, unlike loss-based PTQ, which has to argue that the
loss gradient vanishes. Second, $G_\ell\succeq 0$ as a sum of Gram matrices, which is what
licenses the eigendecomposition of Section~\ref{sec:asp}. If the action head is read as a
Gaussian likelihood with fixed isotropic covariance, $G_\ell$ is also its Fisher information up
to the covariance scale, the usual Gauss--Newton/Fisher
correspondence~\citep{schraudolph2002,martens2014,kunstner2019}. In either reading it is built
from the model's own output and never from ground-truth actions.

\subsection{The block-diagonal reduction and the cross terms it discards}
\label{app:blockdiag}
Section~\ref{sec:formulation} reduces the model-wide curvature $G$ to one term per layer.
Write $\Delta=\mathbb{E}\|\actionfun_\delta-\actionfun\|^2$ for the expected squared action change that
Equation~\ref{eq:obj} approximates. Inserting the first-order deviation
$J\delta=\sum_\ell\sum_{i\in\mathcal{I}_\ell}J^{(i)}_\ell\delta^{(i)}_\ell$ of Equation~\ref{eq:taylor}
and writing it out over the layer and token--step indices gives a double sum,
\begin{equation}
\label{eq:app-blockdiag}
\Delta\approx
\underbrace{\sum_{\ell,i}\mathbb{E}\Big[\delta^{(i)\top}_\ell J^{(i)\top}_\ell J^{(i)}_\ell\delta^{(i)}_\ell\Big]}_{\text{same }(\ell,i)\text{: kept}}
+\underbrace{\sum_{(\ell,i)\neq(\ell',i')}\mathbb{E}\Big[\delta^{(i)\top}_\ell J^{(i)\top}_\ell J^{(i')}_{\ell'}\delta^{(i')}_{\ell'}\Big]}_{\text{different }(\ell,i)\text{: dropped}}.
\end{equation}
The additive-noise model of Section~\ref{sec:formulation} takes the residuals to be zero-mean,
uncorrelated with the Jacobians, and uncorrelated with each other across different $(\ell,i)$.
Under it each cross term factors as
$\tr\big(\mathbb{E}[J^{(i)\top}_\ell J^{(i')}_{\ell'}]\,\mathbb{E}[\delta^{(i')}_{\ell'}\delta^{(i)\top}_\ell]\big)=0$,
and each kept term factors as $\tr\big(\mathbb{E}[J^{(i)\top}_\ell J^{(i)}_\ell]\,\Sigma_\ell\big)$
with $\Sigma_\ell=\mathbb{E}[\delta_\ell\delta_\ell^\top]$ the residual covariance shared by the
token--step pairs of layer $\ell$. Summing the pairs of each layer gives
$\Delta\approx\sum_\ell\tr(G_\ell\Sigma_\ell)=\sum_\ell\mathbb{E}[\delta_\ell^\top G_\ell\delta_\ell]$
with $G_\ell=\mathbb{E}[\sum_i J^{(i)\top}_\ell J^{(i)}_\ell]$, which is Equation~\ref{eq:obj}.

The model is approximate. Adjacent denoising steps see similar activations and can leave
correlated residuals. It is the block-diagonal reduction that K-FAC makes across
layers~\citep{kfac2015} and blockwise PTQ makes across blocks~\citep{brecq2021}, and
Figure~\subfigref{fig:method}{a} is its empirical check, since the kept terms alone predict the measured
single-layer W4A4 action damage at $\rho{=}0.95$ across all $300$ action-expert linears.

The reduction is applied to the curvature of the action map. Layer-wise PTQ
(Section~\ref{sec:prelim-ptq}) applies the same block-diagonal step and then also replaces the
objective by each layer's local output error, which removes all downstream information from the
curvature. The two steps are independent, and Q-WAM takes only the first.

\subsection{Optimality of the top eigenvectors}
\label{app:kyfan}
This appendix states and proves the result that Section~\ref{sec:asp} uses to choose the protected
subspace. Let $V\in\mathbb{R}^{d_\ell\times r}$ have orthonormal columns, let $\Pi=VV^\top$, and let
$\delta=Q_4(\xt_\perp)-\xt_\perp$ be the residual of the 4-bit path in Equation~\ref{eq:asp}. The
deflated weight passes $\Wt^\top(I-\Pi)\delta$ to the output, so the rest of the network sees the
residual $(I-\Pi)\delta$. By Equation~\ref{eq:obj}, the expected per-layer distortion of the 4-bit
path is
\begin{equation*}
\mathbb{E}\big[\delta^\top(I-\Pi)\,\Gt_\ell\,(I-\Pi)\delta\big]
=\tr\big((I-\Pi)\Gt_\ell(I-\Pi)\,\Sigma_\delta\big),\qquad
\Sigma_\delta=\mathbb{E}[\delta\delta^\top].
\end{equation*}

\paragraph{Isotropic rounding error.}
The residual is isotropic when its covariance is a multiple of the identity,
$\Sigma_\delta=\sigma^2I$. Equivalently, the mean squared error along every unit direction $u$ is the same,
$\mathbb{E}[(u^\top\delta)^2]=\sigma^2$, so the rounding error has no preferred direction. The
block Hadamard rotation of Section~\ref{sec:prelim-ptq} spreads outlier channels over all
coordinates, which brings the rounding error close to this state, and Appendix~\ref{app:isotropy}
tests one consequence of it on Fast-WAM. For an isotropic residual, cyclicity of the trace and
$(I-\Pi)^2=I-\Pi$ give
\begin{equation}
\label{eq:aspobj}
\mathbb{E}\big[\delta^\top(I-\Pi)\,\Gt_\ell\,(I-\Pi)\delta\big]
=\sigma^2\tr\!\big(\Gt_\ell(I-\Pi)\big)
=\sigma^2\big(\tr\Gt_\ell-\tr V^\top\Gt_\ell V\big),
\end{equation}
where the last step uses $\tr(\Gt_\ell VV^\top)=\tr(V^\top\Gt_\ell V)$.

\begin{proposition}[Optimal protected subspace]
\label{thm:optimal}
Over all $V\in\mathbb{R}^{d_\ell\times r}$ with orthonormal columns, Equation~\ref{eq:aspobj} reaches its
minimum when the columns of $V$ span the top-$r$ eigenvectors of $\Gt_\ell$, and the minimum equals
$\sigma^2\sum_{j>r}\lambda_j$.
\end{proposition}
\begin{proof}
The term $\tr\Gt_\ell$ does not depend on $V$, so minimizing Equation~\ref{eq:aspobj} over $V$ is the
same as maximizing $\tr(V^\top\Gt_\ell V)$ over matrices with orthonormal columns. The Ky Fan maximum
principle~\citep{fan1949} states that for a symmetric matrix $A$ with eigenvalues
$\lambda_1\ge\dots\ge\lambda_d$,
\begin{equation*}
\max_{V^\top V=I_r}\tr\big(V^\top AV\big)=\sum_{j=1}^{r}\lambda_j,
\end{equation*}
attained when the columns of $V$ span an invariant subspace of the $r$ largest eigenvalues. The
matrix $\Gt_\ell=S^\top G_\ell S$ with $S=\diag(\gamma)H$ is symmetric, so the principle with
$A=\Gt_\ell$ gives the minimum distortion
$\sigma^2(\tr\Gt_\ell-\sum_{j\le r}\lambda_j)=\sigma^2\sum_{j>r}\lambda_j$, attained by the top-$r$
eigenspace.
\end{proof}
Without isotropy, the distortion $\tr\big((I-\Pi)\Gt_\ell(I-\Pi)\Sigma_\delta\big)$ depends on both
$\Gt_\ell$ and $\Sigma_\delta$, and the best subspace need not be the top eigenspace of $\Gt_\ell$.

\subsection{The probe estimator: procedure, cost and accuracy}
\label{app:probe}
This appendix details the estimator of Equation~\ref{eq:ghat}. Throughout,
$J_\ell\in\mathbb{R}^{m\times d_\ell}$ denotes the Jacobian of the action chunk $a\in\mathbb{R}^{m}$
with respect to one token--step activation of layer $\ell$, with the index $i$ suppressed.

\paragraph{Why probes.}
Reverse-mode differentiation computes, in one backward pass, the gradient of one scalar with respect
to every intermediate tensor. For the vector output $a$, one backward pass therefore computes the
vector--Jacobian product $J_\ell^\top w$ for one chosen weight vector $w\in\mathbb{R}^{m}$, by
differentiating the scalar $\langle a,w\rangle$. The exact Gram matrix runs $w$ through the $m$ unit
vectors $e_1,\dots,e_m$,
\begin{equation*}
J_\ell^\top J_\ell=\sum_{k=1}^{m}\big(J_\ell^\top e_k\big)\big(J_\ell^\top e_k\big)^\top ,
\end{equation*}
which costs $m$ backward passes per input, $m{=}448$ on RoboTwin. The probe estimator replaces the
sweep over unit vectors with random vectors $u\sim\mathcal{N}(0,I_m)$. Since
$\mathbb{E}_u[uu^\top]=I_m$,
\begin{equation*}
\mathbb{E}_u\big[(J_\ell^\top u)(J_\ell^\top u)^\top\big]=J_\ell^\top\,\mathbb{E}_u[uu^\top]\,J_\ell
=J_\ell^\top J_\ell ,
\end{equation*}
so each probe gives an unbiased rank-one estimate of the exact sum, and the $P$-probe average of
Equation~\ref{eq:ghat} is unbiased for every $P\ge1$. The number of probes controls only the
variance. The construction is the Hutchinson trace estimator~\citep{hutchinson1989} applied to a
Gram matrix instead of a trace. $G_\ell$ is itself an average over inputs, so probe noise and input
noise average down along the same sum, since $N$ inputs with $P$ probes each contribute $NP$ rank-one
terms for every token--step pair.

\paragraph{Procedure.}
For each calibration input $(o_n,l_n)$, the estimator runs three steps.
\begin{enumerate}[leftmargin=1.5em,itemsep=1pt,topsep=2pt]
\item Encode the observation and the instruction once into the cached conditioning $\z$, and fix
the initial noise $a^{(0)}$ of the sampler to one draw, so that the action chunk $a$ is a
deterministic function of the activations.
\item Run the $T$ Euler steps of Equation~\ref{eq:denoise} with gradients enabled, and record the
input activation $x^{(i)}_\ell$ of every quantized linear layer at every token--step pair $i$. Each
recorded activation is a separate node of the autograd graph, for the reason the Attribution
paragraph below gives.
\item For $p=1,\dots,P$, draw $u_{np}\sim\mathcal{N}(0,I_m)$ and backpropagate the scalar
$\langle a,u_{np}\rangle$ through the recorded graph. Read the gradient
$g^{(i)}_{\ell,np}=J^{(i)\top}_\ell u_{np}$ at every recorded activation, and add
$g^{(i)}_{\ell,np}\,g^{(i)\top}_{\ell,np}$ to a $d_\ell\times d_\ell$ accumulator $A_\ell$.
\end{enumerate}
The graph is kept between probes, so one forward pass serves all $P$ backward passes. After the last
input, $\widehat{G}_\ell=A_\ell/(NP)$, which equals Equation~\ref{eq:ghat}.

\paragraph{Cost.}
Per input, the estimator needs one forward pass and $P{=}12$ backward passes, against $m{=}448$
backward passes for exact Jacobians. Each backward pass returns the vector--Jacobian products of
every layer, so the cost does not grow with the number of layers. Each layer stores one
$d_\ell\times d_\ell$ accumulator, at most $4096\times4096$ on Fast-WAM, independent of $N$. With
$N{=}10{,}919$ inputs, the whole calibration costs $NP\approx1.3\times10^5$ backward passes and runs
once offline. Smoothing and rotation enter afterwards through Equation~\ref{eq:gtilde}, so changing
either transform requires no new calibration. Q-WAM reads three quantities from the estimate, and the
next three paragraphs give the guarantee for each.

\paragraph{Trace.}
$\tr\widehat{G}_\ell=\frac{1}{NP}\sum_{n,p}\|J_\ell(o_n)^\top u_{np}\|^2$ is Hutchinson's
estimator of $\tr(J_\ell^\top J_\ell)$~\citep{hutchinson1989}. For a fixed symmetric positive
semidefinite $A$ and Gaussian probes, the single-probe estimate $u^\top Au$ has variance
$2\|A\|_F^2$, so $NP$ probes give a relative standard deviation of
$\sqrt{2/NP}\,\|A\|_F/\tr A$, and $O(\varepsilon^{-2}\log(1/\eta))$ probes suffice for relative
error $\varepsilon$ at confidence $1-\eta$~\citep{avron2011,roosta2015}. A sharply decaying
spectrum, which Figure~\subfigref{fig:method}{b} shows, puts $\|A\|_F/\tr A$ near its maximum of one,
the worst case for this bound. Equation~\ref{eq:mass} uses the trace only to compare experts
whose masses differ by one to four orders of magnitude, so the estimator's error is far below
the margin it has to resolve, and variance reduction~\citep{hutchpp2021} is unnecessary.

\paragraph{Subspace.}
Collect the $P$ probes of frame $n$ into $U_n\in\mathbb{R}^{m\times P}$ and stack the $N$
Jacobians into $\mathbf{J}=[J_\ell(o_1);\dots;J_\ell(o_N)]\in\mathbb{R}^{Nm\times d_\ell}$, so
that $\mathbf{J}^\top\mathbf{J}/N$ is the empirical $G_\ell$. With the block-diagonal test
matrix $\Omega=\operatorname{blkdiag}(U_1,\dots,U_N)\in\mathbb{R}^{Nm\times NP}$,
\begin{equation*}
NP\,\widehat{G}_\ell=\sum_n J_\ell(o_n)^\top U_nU_n^\top J_\ell(o_n)
=\mathbf{J}^\top\Omega\,\Omega^\top\mathbf{J},
\end{equation*}
so $\widehat{G}_\ell$ is, up to scale, the Gram matrix of the sketch $\mathbf{J}^\top\Omega$.
This sketch is the randomized range finder~\citep{hmt2011} applied to $\mathbf{J}^\top$ with $NP$ test
vectors. The range finder recovers the dominant rank-$r$ singular subspace once the number of
test vectors exceeds $r$ by a small oversampling margin, with an error set by the trailing
singular values. The count that has to exceed $r$ is therefore
$NP=10{,}919\times12\approx1.3\times10^5$, against $d_\ell\approx10^3$, and not $P$ alone; $P$
sets only how much of each frame's Jacobian is seen. The guarantee of the range finder~\citep{hmt2011} is stated
for a dense Gaussian test matrix and ours is block-diagonal, so we also rely on the unbiasedness
of $\widehat{G}_\ell$ and on $NP\gg d_\ell$. Choosing $P<r$ per frame allocates a fixed
backward-pass budget across frames rather than within one, which is the right allocation
because $G_\ell$ averages over frames. Exact per-frame Jacobians would cost $m/P\approx37\times$
more per frame for a per-frame precision that the average then dilutes.

\paragraph{Basis.}
Gaussian probes have a rotationally invariant law, so the estimator's error is not tied to any
coordinate system of action space. Rademacher probes have lower variance for the trace and are
the usual choice when only a trace is wanted~\citep{avron2011}, but their law singles out the
coordinate axes, which is the wrong trade when the estimate is consumed as an eigenbasis.

\paragraph{Attribution.}
One implementation detail matters. When several layers consume the same tensor, as the $q$, $k$
and $v$ projections of one attention block do, the gradient stored on that tensor is the sum
over all its consumers, and reading it would credit each layer with the others' sensitivity. We
therefore clone each layer's input so that it is a distinct node in the autograd graph, and read
$J_\ell^\top u$ from the clone. An earlier version of our AOGs omitted this and was recomputed;
the shipped AOGs use per-consumer attribution.

\subsection{Checking the isotropic noise model}
\label{app:isotropy}
Proposition~\ref{thm:optimal} and the action mass of Equation~\ref{eq:mass} assume that the rotated
4-bit residual is isotropic, with no preferred direction. An isotropic residual places the fraction
$r/d_\ell$ of its expected squared norm $\mathbb{E}\|\delta\|^2$ inside any $r$-dimensional subspace. We test this consequence on twelve
action-expert layers of Fast-WAM, the ten linear layers of the first block and the query and key
projections of the second, and report medians. The residual of $Q_4$ places $3.1\%$ of its squared norm inside $\spn(V)$, which matches the chance rate $r/d_\ell$ and the share inside a random subspace of
the same rank. The same
$3.1\%$ of the squared norm carries $97\%$ of the action damage $\delta^\top\Gt_\ell\delta$, against $2.9\%$
for the random subspace, and deflation removes it at the output up to a second-order product of
rounding errors. Removing the protected component also shrinks the dynamic range that sets
$\varsigma$ in Equation~\ref{eq:quant}, by a median factor of $0.77$ against $0.99$ for the random
subspace, so the 4-bit residual contracts in every direction.

\section{Implementation Details}
\label{app1}

\subsection{Quantization settings and kernels}
\label{app:kernels}
All quantizers are symmetric and round to nearest, with one scale for each group of consecutive
input channels. For INT4 we use a group size of 32 for both weights and activations. Each weight
group stores 4-bit codes in $[-8,7]$, and each activation group is quantized at runtime, per token,
with a scale equal to its absolute maximum divided by $7$. On Blackwell GPUs we use NVFP4 instead,
which stores E2M1 values with an FP8 (E4M3) scale for every group of 16 elements, for both weights
and activations; the weights carry an additional FP32 scale per tensor.

Before the rotation, we apply SmoothQuant smoothing with the per-channel factor
\begin{equation}
s_j=\frac{\max|X_{j}|^{\alpha}}{\max|W_{:,j}|^{1-\alpha}},
\end{equation}
where $\max|X_{j}|$ is the absolute maximum of input channel $j$ over the calibration inputs and
$\max|W_{:,j}|$ is the absolute maximum of the matching weight column. We divide the activation by
$s$ and multiply the weight by $s$, so the layer output is unchanged before quantization. In the
RoboTwin experiments we use a smoothing factor of $\alpha=0.5$ for every layer of all three models.
The activation maxima are collected on 50 RoboTwin episodes, one per task. The block Hadamard
rotation uses the largest power of two that divides the input dimension, up to 1024, and the
protected subspace has rank $r=32$ with its branch kept in 16-bit precision.

The action-block latency reported in Section~\ref{sec:exp} is measured with 4-bit kernels
written in Triton.
A quantized layer runs two of them. The first applies the smoothing, the block Hadamard rotation,
the projection onto the protected subspace and the activation quantization, and the second
multiplies the packed 4-bit weights, adding the protected branch and the bias in its epilogue. The
weights stay packed at 4 bits throughout and are the exported codes of the evaluated checkpoint,
and the activations are quantized by the same rule as in evaluation; the fused implementation
differs from it only in the order in which it accumulates in floating point.  We time the quantized
layers of one action chunk since the ASP has been applied to that expert for the corresponding models. The
precisions are replayed as CUDA graphs with the same layer layout. Each reported value is the
median over repeated runs on one L40S GPU.

\subsection{Cost of the 16-bit branch}
\label{app:overhead}
\begin{wrapfigure}{r}{0.34\textwidth}
\vspace{-\baselineskip}
\centering
\includegraphics[width=\linewidth]{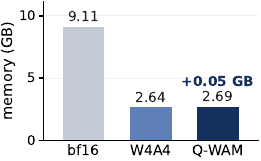}
\caption{\raggedright ImageWAM's memory comparison for the ASP overhead.}
\label{fig:overhead}
\vspace{-\baselineskip}
\end{wrapfigure}
The 16-bit branch of ASP adds little to the 4-bit model. Figure~\ref{fig:overhead} compares the
memory of ImageWAM's quantized blocks: $9.11$\,GB in bf16, $2.64$\,GB for the 4-bit model without the
branch (smoothing, rotation and group-wise W4A4), and $2.69$\,GB for Q-WAM. The branch therefore
adds $0.05$\,GB, $1.9\%$ of the 4-bit model, while the model as a whole stays $3.4\times$ smaller
than bf16. Across the three models, the branch adds $0.05$--$0.16$\,GB and $0.08$--$0.26$ bits per
weight (Table~\ref{tab:ablation}). Its computation is small as well: both paths are dense matrix
multiplications, and with the branch included the quantized action-token layers still run
$1.06$--$1.56\times$ faster than bf16 (Figure~\ref{fig:latency}).

\subsection{Real-world setup}
\label{app:realsetup}
\begin{figure}
    \centering
   \includegraphics[width=0.9\linewidth]{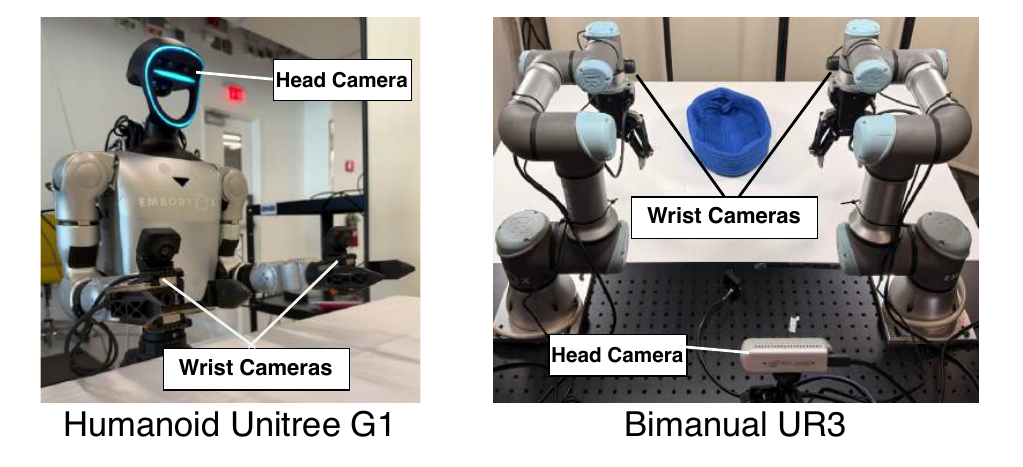}
    \caption{Real-world experiment setup.}
    \label{fig:placeholder}
\end{figure}
Figure~\ref{fig:placeholder} shows the two embodiments of the real-world experiments. The first is a
Unitree G1 humanoid with parallel-jaw grippers, and the second is a bimanual UR3 robot. Each
embodiment carries four cameras. Table~\ref{tab:realworld} lists the tasks run on each.

\section{Supplementary Experiments}
\label{app:5}

\subsection{SVDQuant as the weight quantizer}
\label{app:weightside}
SVDQuant protects weights rather than activations. It keeps a rank-32 16-bit branch built from the
leading singular vectors of the smoothed weight and quantizes the remainder, which works well on
image diffusion models. Table~\ref{tab:weightside} tests whether this weight protection
adds to Q-WAM. ASP is held fixed on the activation side, and the weights are quantized either
group-wise to INT4, as in Q-WAM, or with SVDQuant's decomposition. The SVDQuant branch costs
0.25 to 0.33 additional BPW and 0.14 to 0.24\,GB of memory, but it does not improve the success
rate. The average changes by $+0.21$ points on Fast-WAM and $+0.14$ points on LingBot-VA, both
smaller than the variation between calibration sets (Appendix~\ref{app:calib}), and it
falls by $2.61$ points on ImageWAM. SVDQuant selects the protected directions by the singular values
of the weight, a criterion that ignores how the action depends on each layer. In a WAM, this
dependence decides which quantization errors matter. Figure~\subfigref{fig:method}{a} shows that
per-layer reconstruction error does not predict the damage to the action, whereas the AOG does.
Therefore, the weight directions that SVDQuant keeps in 16 bits are not the ones the action
depends on, and group-wise INT4 is the cheaper choice for the weights of Q-WAM.

\begin{table}[t]\centering\small
\caption{Q-WAM with SVDQuant as the weight quantizer. ASP is held fixed on the activation
side and only the weight side varies, between group-wise INT4 and SVDQuant's low-rank
decomposition.}
\label{tab:weightside}
\begin{tabular}{lllccccc}
\toprule
model & activation & weight & BPW & clean & randomized & Avg. & Mem (GB) \\
\midrule
Fast-WAM    & ASP & group INT4 & \textbf{4.62} & 91.62 & 89.94 & 90.78 & \textbf{3.44} \\
           & ASP & SVDQuant   & 4.95 & 91.44 & 90.54 & 90.99 & 3.68 \\
\midrule
ImageWAM   & ASP & group INT4 & \textbf{4.58} & 93.00 & 92.94 & 92.97 & \textbf{2.69} \\
           & ASP & SVDQuant   & 4.83 & 90.56 & 90.16 & 90.36 & 2.83 \\
\midrule
LingBot-VA & ASP & group INT4 & \textbf{4.76} & 90.20 & 88.92 & 89.56 & \textbf{3.27} \\
           & ASP & SVDQuant   & 5.03 & 90.52 & 88.88 & 89.70 & 3.43 \\
\bottomrule
\end{tabular}
\end{table}

\subsection{Where the action mass sits within each expert}
\label{app:aogmap}

Section~\ref{sec:where} scores whole experts by their action mass $\mu_E$
(Equation~\ref{eq:mass}). Figure~\ref{fig:aogmap} resolves the same quantity to single layers on
Fast-WAM. Each cell is the trace $\tr(\Gt_\ell)$ of one linear layer as a share of the model's
total, arranged by layer type and block, and both experts share one color scale. The action expert
holds $89.3\%$ of the total. Within it the mass concentrates in the cross-attention output
projections, which carry $57.2\%$ of the expert's mass, and cross-attention as a whole carries
$82.0\%$ against $2.0\%$ for self-attention. Across the $300$ action layers the trace spans more
than six orders of magnitude. The video expert holds the remaining $10.7\%$ and spreads it more
evenly, led by the second feed-forward projection at $32.1\%$ and the cross-attention output
projection at $25.5\%$. Eight layers of the video expert's last block carry exactly zero mass,
because no later block reads their outputs; the action expert reads that block only through its
self-attention key and value projections, which are the two layers of the block that carry mass.

\begin{figure}[ht]
\centering
\includegraphics[width=\textwidth]{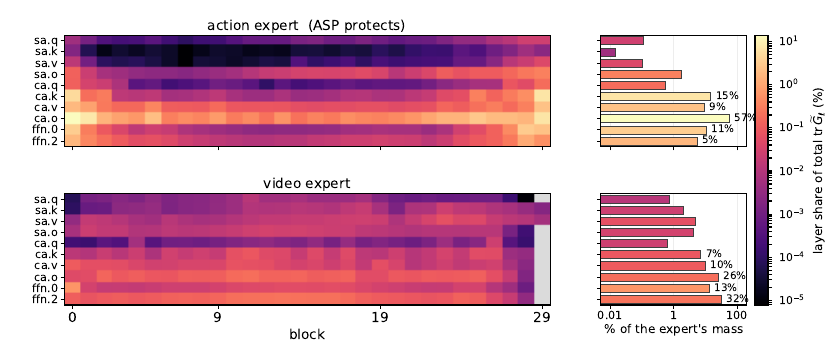}
\caption{Share of Fast-WAM's action mass carried by each linear layer, for the action expert (top)
and the video expert (bottom) on one logarithmic color scale. Rows are layer types and columns are
blocks. The bars on the right give each layer type's share of its own expert's mass. Gray cells
have exactly zero mass.}
\label{fig:aogmap}
\end{figure}

\subsection{Sensitivity to the calibration set}
\label{app:calib}
The AOG and the protected subspaces are estimated from 50 calibration episodes, one per RoboTwin task.
Table~\ref{tab:calib} measures how much the result depends on which episodes are drawn. On ImageWAM, we
draw three further sets of 50 episodes with random seeds 7, 13 and 21, in addition to the original set,
which was drawn with seed 42. For each of the four sets, we recompute the AOG and the rank-32 protected subspaces, export a new checkpoint, and
evaluate it on all 50 tasks with 100 episodes per task under both the clean and the randomized
conditions.

The four checkpoints reach an average success rate of $92.41\%$ with a standard deviation of $0.49$
points. The largest gap between two calibration sets is $1.12$ points, between seeds 7 and 13.
The protected subspaces also barely change between sets. Any two sets
share $97$--$99\%$ of their rank-32 subspace, measured as $\|V_1^\top V_2\|_F^2/r$ and averaged over
layers, while a random subspace of the same rank shares about $2\%$. The protected subspace, and with it
the success rate, is therefore insensitive to the choice of calibration episodes.

\begin{table}[h]\centering\small
\caption{Sensitivity of ImageWAM to the calibration set. Each row recomputes the AOG, the protected
subspaces and the checkpoint from a different set of calibration episodes, one per task, and evaluates
it on the full RoboTwin suite.}
\label{tab:calib}
\begin{tabular}{lccc}
\toprule
Calibration seed & Clean & Randomized & Avg. \\
\midrule
42 (original set) & 93.08 & 92.28 & 92.68 \\
7                 & 92.62 & 93.08 & 92.85 \\
13                & 91.48 & 91.98 & 91.73 \\
21                & 92.34 & 92.44 & 92.39 \\
\midrule
Mean $\pm$ std. & $92.38\pm0.67$ & $92.44\pm0.46$ & $92.41\pm0.49$ \\
\bottomrule
\end{tabular}
\end{table}

\subsection{Per-task RoboTwin results}
\label{app:3}
\label{app:pertask}
RoboTwin 2.0 has 50 tasks. Tables~\ref{tab:pertask-fastwam}, \ref{tab:pertask-imagewam}
and~\ref{tab:pertask-lingbot} extend Table~\ref{tab:main} task by task for Fast-WAM, ImageWAM and
LingBot-VA. They report the success rate of every method on each task under the clean and the
domain-randomized conditions.

\begin{table}[p]
\centering
\caption{Per-task success rates (\%) of Fast-WAM on RoboTwin 2.0 for every method in Table~\ref{tab:main}, under the clean (Clean) and domain-randomized (Rand.) conditions. The best quantized average is shown in bold.}
\label{tab:pertask-fastwam}
\footnotesize
\setlength{\tabcolsep}{2.4pt}
\begin{tabular}{lcccccccccc}
\toprule
 & \multicolumn{2}{c}{bf16} & \multicolumn{2}{c}{SmoothQuant} & \multicolumn{2}{c}{ViDiT-Q} & \multicolumn{2}{c}{SVDQuant} & \multicolumn{2}{c}{\textbf{Q-WAM}} \\
\cmidrule(lr){2-3} \cmidrule(lr){4-5} \cmidrule(lr){6-7} \cmidrule(lr){8-9} \cmidrule(lr){10-11}
Task & Clean & Rand. & Clean & Rand. & Clean & Rand. & Clean & Rand. & Clean & Rand. \\
\midrule
\texttt{adjust\_bottle} & 100 & 100 & 47 & 21 & 100 & 100 & 100 & 100 & 100 & 100 \\
\texttt{beat\_block\_hammer} & 97 & 98 & 18 & 10 & 94 & 94 & 98 & 94 & 100 & 96 \\
\texttt{blocks\_ranking\_rgb} & 100 & 100 & 1 & 0 & 96 & 97 & 100 & 99 & 100 & 100 \\
\texttt{blocks\_ranking\_size} & 95 & 93 & 1 & 0 & 83 & 90 & 89 & 91 & 90 & 92 \\
\texttt{click\_alarmclock} & 99 & 100 & 96 & 92 & 100 & 100 & 100 & 100 & 100 & 100 \\
\texttt{click\_bell} & 100 & 100 & 100 & 99 & 100 & 100 & 100 & 100 & 100 & 100 \\
\texttt{dump\_bin\_bigbin} & 98 & 96 & 35 & 20 & 85 & 87 & 88 & 87 & 96 & 98 \\
\texttt{grab\_roller} & 100 & 100 & 43 & 25 & 100 & 100 & 100 & 100 & 100 & 100 \\
\texttt{handover\_block} & 94 & 80 & 0 & 0 & 9 & 16 & 79 & 69 & 94 & 74 \\
\texttt{handover\_mic} & 99 & 100 & 0 & 0 & 72 & 67 & 92 & 91 & 100 & 97 \\
\texttt{hanging\_mug} & 57 & 58 & 5 & 1 & 44 & 46 & 51 & 49 & 64 & 54 \\
\texttt{lift\_pot} & 99 & 100 & 21 & 6 & 90 & 95 & 98 & 92 & 100 & 99 \\
\texttt{move\_can\_pot} & 91 & 91 & 4 & 0 & 79 & 70 & 94 & 92 & 85 & 88 \\
\texttt{move\_pillbottle\_pad} & 99 & 99 & 2 & 6 & 84 & 90 & 93 & 97 & 99 & 100 \\
\texttt{move\_playingcard\_away} & 100 & 100 & 22 & 12 & 97 & 99 & 100 & 100 & 100 & 100 \\
\texttt{move\_stapler\_pad} & 81 & 63 & 0 & 0 & 54 & 35 & 72 & 59 & 72 & 61 \\
\texttt{open\_laptop} & 98 & 100 & 63 & 53 & 98 & 97 & 98 & 98 & 98 & 99 \\
\texttt{open\_microwave} & 80 & 47 & 0 & 1 & 60 & 42 & 56 & 37 & 70 & 33 \\
\texttt{pick\_diverse\_bottles} & 83 & 86 & 7 & 2 & 82 & 78 & 91 & 90 & 84 & 83 \\
\texttt{pick\_dual\_bottles} & 100 & 96 & 10 & 3 & 95 & 94 & 99 & 98 & 100 & 98 \\
\texttt{place\_a2b\_left} & 95 & 92 & 22 & 10 & 91 & 91 & 93 & 92 & 96 & 96 \\
\texttt{place\_a2b\_right} & 97 & 94 & 26 & 15 & 92 & 91 & 96 & 95 & 97 & 88 \\
\texttt{place\_bread\_basket} & 93 & 94 & 28 & 22 & 90 & 84 & 90 & 92 & 94 & 95 \\
\texttt{place\_bread\_skillet} & 93 & 87 & 9 & 1 & 86 & 80 & 95 & 90 & 93 & 94 \\
\texttt{place\_burger\_fries} & 94 & 98 & 19 & 3 & 90 & 96 & 93 & 95 & 97 & 96 \\
\texttt{place\_can\_basket} & 66 & 65 & 1 & 2 & 62 & 56 & 70 & 66 & 65 & 70 \\
\texttt{place\_cans\_plasticbox} & 98 & 99 & 1 & 0 & 93 & 93 & 95 & 97 & 100 & 97 \\
\texttt{place\_container\_plate} & 96 & 94 & 38 & 33 & 95 & 97 & 99 & 97 & 98 & 100 \\
\texttt{place\_dual\_shoes} & 92 & 89 & 0 & 0 & 93 & 82 & 86 & 88 & 88 & 91 \\
\texttt{place\_empty\_cup} & 100 & 100 & 59 & 35 & 100 & 99 & 99 & 100 & 100 & 100 \\
\texttt{place\_fan} & 95 & 96 & 0 & 0 & 69 & 74 & 76 & 69 & 95 & 94 \\
\texttt{place\_mouse\_pad} & 88 & 89 & 6 & 3 & 74 & 75 & 82 & 81 & 84 & 83 \\
\texttt{place\_object\_basket} & 84 & 87 & 9 & 6 & 89 & 89 & 86 & 79 & 80 & 82 \\
\texttt{place\_object\_scale} & 88 & 93 & 7 & 1 & 86 & 85 & 88 & 89 & 87 & 92 \\
\texttt{place\_object\_stand} & 88 & 91 & 16 & 14 & 96 & 93 & 89 & 91 & 90 & 95 \\
\texttt{place\_phone\_stand} & 98 & 98 & 9 & 7 & 93 & 94 & 96 & 98 & 99 & 98 \\
\texttt{place\_shoe} & 96 & 99 & 20 & 10 & 86 & 96 & 95 & 97 & 95 & 97 \\
\texttt{press\_stapler} & 92 & 97 & 96 & 87 & 96 & 95 & 98 & 97 & 92 & 93 \\
\texttt{put\_bottles\_dustbin} & 91 & 90 & 0 & 0 & 29 & 33 & 78 & 73 & 88 & 82 \\
\texttt{put\_object\_cabinet} & 90 & 88 & 0 & 0 & 34 & 28 & 85 & 88 & 89 & 86 \\
\texttt{rotate\_qrcode} & 91 & 91 & 0 & 2 & 89 & 77 & 90 & 85 & 91 & 86 \\
\texttt{scan\_object} & 91 & 89 & 39 & 25 & 81 & 77 & 84 & 85 & 92 & 91 \\
\texttt{shake\_bottle} & 100 & 100 & 90 & 82 & 100 & 100 & 100 & 100 & 100 & 100 \\
\texttt{shake\_bottle\_horizontally} & 100 & 100 & 91 & 84 & 99 & 100 & 100 & 100 & 100 & 100 \\
\texttt{stack\_blocks\_three} & 96 & 95 & 0 & 0 & 67 & 80 & 94 & 91 & 95 & 96 \\
\texttt{stack\_blocks\_two} & 100 & 100 & 3 & 2 & 100 & 97 & 100 & 100 & 100 & 99 \\
\texttt{stack\_bowls\_three} & 85 & 90 & 0 & 0 & 78 & 66 & 76 & 66 & 81 & 83 \\
\texttt{stack\_bowls\_two} & 95 & 97 & 2 & 1 & 92 & 89 & 96 & 94 & 94 & 95 \\
\texttt{stamp\_seal} & 90 & 94 & 12 & 7 & 63 & 55 & 72 & 71 & 89 & 86 \\
\texttt{turn\_switch} & 62 & 64 & 21 & 21 & 65 & 55 & 57 & 60 & 60 & 60 \\
\midrule
Average & 92.28 & 91.34 & 21.98 & 16.48 & 82.00 & 80.48 & 89.12 & 87.38 & \textbf{91.62} & \textbf{89.94} \\
\bottomrule
\end{tabular}

\end{table}

\begin{table}[p]
\centering
\caption{Per-task success rates (\%) of ImageWAM on RoboTwin 2.0 for every method in Table~\ref{tab:main}, under the clean (Clean) and domain-randomized (Rand.) conditions. The best quantized average is shown in bold.}
\label{tab:pertask-imagewam}
\footnotesize
\setlength{\tabcolsep}{2.4pt}
\begin{tabular}{lcccccccccc}
\toprule
 & \multicolumn{2}{c}{bf16} & \multicolumn{2}{c}{SmoothQuant} & \multicolumn{2}{c}{ViDiT-Q} & \multicolumn{2}{c}{SVDQuant} & \multicolumn{2}{c}{\textbf{Q-WAM}} \\
\cmidrule(lr){2-3} \cmidrule(lr){4-5} \cmidrule(lr){6-7} \cmidrule(lr){8-9} \cmidrule(lr){10-11}
Task & Clean & Rand. & Clean & Rand. & Clean & Rand. & Clean & Rand. & Clean & Rand. \\
\midrule
\texttt{adjust\_bottle} & 100 & 100 & 96 & 97 & 99 & 99 & 100 & 100 & 100 & 100 \\
\texttt{beat\_block\_hammer} & 100 & 98 & 42 & 27 & 91 & 97 & 95 & 97 & 98 & 92 \\
\texttt{blocks\_ranking\_rgb} & 96 & 99 & 89 & 91 & 73 & 82 & 97 & 98 & 99 & 99 \\
\texttt{blocks\_ranking\_size} & 92 & 97 & 67 & 71 & 60 & 82 & 95 & 93 & 93 & 96 \\
\texttt{click\_alarmclock} & 100 & 100 & 100 & 100 & 100 & 99 & 100 & 99 & 99 & 100 \\
\texttt{click\_bell} & 100 & 100 & 100 & 100 & 100 & 100 & 99 & 100 & 100 & 100 \\
\texttt{dump\_bin\_bigbin} & 96 & 93 & 67 & 60 & 94 & 82 & 88 & 86 & 98 & 88 \\
\texttt{grab\_roller} & 100 & 100 & 100 & 100 & 100 & 100 & 100 & 100 & 100 & 100 \\
\texttt{handover\_block} & 94 & 96 & 17 & 8 & 97 & 91 & 64 & 64 & 96 & 95 \\
\texttt{handover\_mic} & 100 & 99 & 81 & 85 & 56 & 69 & 96 & 97 & 98 & 98 \\
\texttt{hanging\_mug} & 74 & 77 & 40 & 29 & 60 & 67 & 56 & 70 & 75 & 76 \\
\texttt{lift\_pot} & 100 & 100 & 35 & 37 & 83 & 89 & 95 & 90 & 100 & 100 \\
\texttt{move\_can\_pot} & 96 & 98 & 20 & 28 & 94 & 98 & 83 & 75 & 96 & 100 \\
\texttt{move\_pillbottle\_pad} & 99 & 100 & 20 & 25 & 92 & 89 & 81 & 85 & 97 & 99 \\
\texttt{move\_playingcard\_away} & 100 & 99 & 92 & 97 & 97 & 96 & 97 & 97 & 100 & 99 \\
\texttt{move\_stapler\_pad} & 66 & 67 & 16 & 9 & 36 & 44 & 42 & 34 & 51 & 65 \\
\texttt{open\_laptop} & 97 & 98 & 96 & 96 & 95 & 98 & 95 & 97 & 98 & 99 \\
\texttt{open\_microwave} & 97 & 94 & 27 & 29 & 13 & 14 & 6 & 22 & 82 & 73 \\
\texttt{pick\_diverse\_bottles} & 87 & 89 & 47 & 43 & 89 & 83 & 89 & 83 & 86 & 89 \\
\texttt{pick\_dual\_bottles} & 99 & 96 & 39 & 41 & 88 & 84 & 95 & 93 & 100 & 98 \\
\texttt{place\_a2b\_left} & 95 & 96 & 76 & 85 & 87 & 92 & 93 & 94 & 95 & 98 \\
\texttt{place\_a2b\_right} & 99 & 96 & 89 & 85 & 91 & 96 & 95 & 93 & 99 & 95 \\
\texttt{place\_bread\_basket} & 94 & 96 & 67 & 61 & 88 & 86 & 89 & 90 & 97 & 93 \\
\texttt{place\_bread\_skillet} & 91 & 90 & 59 & 55 & 76 & 86 & 91 & 86 & 93 & 94 \\
\texttt{place\_burger\_fries} & 100 & 99 & 78 & 78 & 91 & 91 & 91 & 90 & 100 & 99 \\
\texttt{place\_can\_basket} & 77 & 62 & 28 & 20 & 66 & 58 & 56 & 59 & 80 & 71 \\
\texttt{place\_cans\_plasticbox} & 100 & 97 & 48 & 38 & 80 & 87 & 84 & 75 & 98 & 97 \\
\texttt{place\_container\_plate} & 99 & 97 & 85 & 88 & 92 & 97 & 95 & 95 & 98 & 99 \\
\texttt{place\_dual\_shoes} & 77 & 86 & 28 & 40 & 25 & 33 & 65 & 69 & 71 & 83 \\
\texttt{place\_empty\_cup} & 100 & 100 & 84 & 87 & 100 & 100 & 99 & 100 & 100 & 100 \\
\texttt{place\_fan} & 95 & 96 & 26 & 34 & 54 & 65 & 69 & 68 & 93 & 95 \\
\texttt{place\_mouse\_pad} & 87 & 93 & 29 & 47 & 56 & 58 & 64 & 61 & 85 & 87 \\
\texttt{place\_object\_basket} & 78 & 82 & 68 & 63 & 80 & 75 & 82 & 81 & 82 & 87 \\
\texttt{place\_object\_scale} & 93 & 97 & 64 & 73 & 82 & 87 & 81 & 85 & 91 & 98 \\
\texttt{place\_object\_stand} & 100 & 98 & 75 & 73 & 95 & 90 & 92 & 96 & 95 & 95 \\
\texttt{place\_phone\_stand} & 100 & 99 & 86 & 90 & 100 & 99 & 88 & 96 & 100 & 98 \\
\texttt{place\_shoe} & 95 & 99 & 78 & 81 & 79 & 88 & 94 & 97 & 97 & 98 \\
\texttt{press\_stapler} & 95 & 100 & 98 & 98 & 100 & 97 & 97 & 100 & 96 & 100 \\
\texttt{put\_bottles\_dustbin} & 95 & 92 & 11 & 3 & 67 & 59 & 89 & 66 & 98 & 87 \\
\texttt{put\_object\_cabinet} & 88 & 95 & 20 & 16 & 63 & 79 & 70 & 73 & 86 & 87 \\
\texttt{rotate\_qrcode} & 86 & 86 & 53 & 47 & 80 & 83 & 85 & 79 & 90 & 86 \\
\texttt{scan\_object} & 93 & 89 & 26 & 22 & 70 & 80 & 86 & 88 & 91 & 95 \\
\texttt{shake\_bottle} & 100 & 100 & 93 & 97 & 100 & 100 & 99 & 99 & 100 & 100 \\
\texttt{shake\_bottle\_horizontally} & 100 & 100 & 93 & 98 & 100 & 100 & 100 & 99 & 100 & 100 \\
\texttt{stack\_blocks\_three} & 96 & 98 & 24 & 45 & 64 & 77 & 93 & 84 & 97 & 97 \\
\texttt{stack\_blocks\_two} & 100 & 100 & 85 & 74 & 86 & 81 & 99 & 98 & 100 & 100 \\
\texttt{stack\_bowls\_three} & 70 & 82 & 35 & 38 & 51 & 70 & 64 & 74 & 84 & 79 \\
\texttt{stack\_bowls\_two} & 92 & 97 & 66 & 72 & 82 & 88 & 90 & 93 & 95 & 97 \\
\texttt{stamp\_seal} & 78 & 83 & 21 & 14 & 41 & 56 & 69 & 68 & 92 & 90 \\
\texttt{turn\_switch} & 75 & 80 & 68 & 73 & 75 & 83 & 75 & 75 & 81 & 76 \\
\midrule
Average & 92.82 & 93.70 & 59.04 & 59.36 & 78.76 & 82.08 & 84.34 & 84.22 & \textbf{93.00} & \textbf{92.94} \\
\bottomrule
\end{tabular}
\end{table}

\begin{table}[p]
\centering
\caption{Per-task success rates (\%) of LingBot-VA on RoboTwin 2.0 for every method in Table~\ref{tab:main}, under the clean (Clean) and domain-randomized (Rand.) conditions. The best quantized average is shown in bold.}
\label{tab:pertask-lingbot}
\footnotesize
\setlength{\tabcolsep}{2.4pt}
\begin{tabular}{lcccccccccc}
\toprule
 & \multicolumn{2}{c}{bf16} & \multicolumn{2}{c}{SmoothQuant} & \multicolumn{2}{c}{ViDiT-Q} & \multicolumn{2}{c}{SVDQuant} & \multicolumn{2}{c}{\textbf{Q-WAM}} \\
\cmidrule(lr){2-3} \cmidrule(lr){4-5} \cmidrule(lr){6-7} \cmidrule(lr){8-9} \cmidrule(lr){10-11}
Task & Clean & Rand. & Clean & Rand. & Clean & Rand. & Clean & Rand. & Clean & Rand. \\
\midrule
\texttt{adjust\_bottle} & 96 & 96 & 96 & 90 & 95 & 95 & 100 & 96 & 98 & 94 \\
\texttt{beat\_block\_hammer} & 94 & 100 & 90 & 86 & 90 & 83 & 98 & 94 & 100 & 92 \\
\texttt{blocks\_ranking\_rgb} & 96 & 94 & 82 & 58 & 76 & 65 & 88 & 88 & 94 & 96 \\
\texttt{blocks\_ranking\_size} & 96 & 82 & 66 & 46 & 81 & 53 & 86 & 62 & 100 & 80 \\
\texttt{click\_alarmclock} & 100 & 100 & 100 & 98 & 100 & 100 & 100 & 100 & 100 & 100 \\
\texttt{click\_bell} & 100 & 100 & 100 & 100 & 100 & 98 & 100 & 100 & 100 & 100 \\
\texttt{dump\_bin\_bigbin} & 92 & 94 & 92 & 78 & 90 & 92 & 88 & 98 & 98 & 94 \\
\texttt{grab\_roller} & 100 & 100 & 100 & 100 & 97 & 97 & 100 & 100 & 100 & 100 \\
\texttt{handover\_block} & 100 & 90 & 58 & 36 & 56 & 44 & 88 & 80 & 100 & 90 \\
\texttt{handover\_mic} & 94 & 98 & 82 & 64 & 71 & 66 & 74 & 84 & 82 & 90 \\
\texttt{hanging\_mug} & 18 & 20 & 18 & 16 & 14 & 10 & 22 & 40 & 18 & 34 \\
\texttt{lift\_pot} & 100 & 100 & 98 & 92 & 97 & 92 & 98 & 88 & 100 & 100 \\
\texttt{move\_can\_pot} & 94 & 94 & 82 & 86 & 87 & 74 & 82 & 72 & 96 & 82 \\
\texttt{move\_pillbottle\_pad} & 100 & 100 & 92 & 82 & 96 & 89 & 96 & 86 & 100 & 100 \\
\texttt{move\_playingcard\_away} & 100 & 98 & 100 & 90 & 95 & 91 & 100 & 96 & 98 & 100 \\
\texttt{move\_stapler\_pad} & 68 & 70 & 48 & 28 & 40 & 21 & 50 & 42 & 62 & 66 \\
\texttt{open\_laptop} & 100 & 96 & 96 & 88 & 98 & 94 & 94 & 92 & 96 & 94 \\
\texttt{open\_microwave} & 46 & 90 & 20 & 34 & 26 & 71 & 22 & 48 & 28 & 76 \\
\texttt{pick\_diverse\_bottles} & 94 & 88 & 88 & 56 & 87 & 56 & 80 & 86 & 96 & 88 \\
\texttt{pick\_dual\_bottles} & 100 & 100 & 94 & 74 & 99 & 72 & 98 & 90 & 100 & 96 \\
\texttt{place\_a2b\_left} & 98 & 90 & 94 & 86 & 89 & 83 & 94 & 84 & 98 & 90 \\
\texttt{place\_a2b\_right} & 86 & 92 & 90 & 82 & 89 & 76 & 94 & 94 & 100 & 92 \\
\texttt{place\_bread\_basket} & 94 & 96 & 86 & 84 & 84 & 83 & 84 & 92 & 96 & 92 \\
\texttt{place\_bread\_skillet} & 98 & 90 & 90 & 78 & 83 & 74 & 94 & 86 & 90 & 90 \\
\texttt{place\_burger\_fries} & 98 & 98 & 94 & 78 & 80 & 73 & 98 & 98 & 94 & 100 \\
\texttt{place\_can\_basket} & 82 & 84 & 72 & 50 & 68 & 71 & 74 & 58 & 86 & 72 \\
\texttt{place\_cans\_plasticbox} & 100 & 98 & 84 & 68 & 96 & 81 & 92 & 88 & 100 & 96 \\
\texttt{place\_container\_plate} & 100 & 98 & 100 & 92 & 96 & 96 & 94 & 90 & 100 & 98 \\
\texttt{place\_dual\_shoes} & 94 & 88 & 58 & 50 & 41 & 34 & 74 & 78 & 92 & 86 \\
\texttt{place\_empty\_cup} & 100 & 100 & 100 & 100 & 99 & 100 & 100 & 100 & 100 & 100 \\
\texttt{place\_fan} & 90 & 88 & 72 & 52 & 70 & 62 & 70 & 72 & 100 & 84 \\
\texttt{place\_mouse\_pad} & 88 & 92 & 72 & 56 & 76 & 62 & 86 & 82 & 96 & 92 \\
\texttt{place\_object\_basket} & 90 & 76 & 80 & 60 & 84 & 66 & 82 & 86 & 84 & 78 \\
\texttt{place\_object\_scale} & 100 & 96 & 86 & 74 & 87 & 82 & 100 & 88 & 100 & 90 \\
\texttt{place\_object\_stand} & 98 & 92 & 92 & 86 & 98 & 83 & 96 & 80 & 100 & 80 \\
\texttt{place\_phone\_stand} & 98 & 94 & 96 & 86 & 87 & 86 & 98 & 90 & 98 & 94 \\
\texttt{place\_shoe} & 94 & 98 & 88 & 84 & 95 & 85 & 92 & 90 & 98 & 98 \\
\texttt{press\_stapler} & 78 & 84 & 94 & 96 & 97 & 93 & 100 & 96 & 84 & 90 \\
\texttt{put\_bottles\_dustbin} & 74 & 74 & 16 & 28 & 26 & 39 & 34 & 50 & 76 & 72 \\
\texttt{put\_object\_cabinet} & 92 & 76 & 52 & 54 & 54 & 45 & 84 & 58 & 74 & 80 \\
\texttt{rotate\_qrcode} & 92 & 90 & 90 & 76 & 84 & 82 & 94 & 80 & 96 & 90 \\
\texttt{scan\_object} & 96 & 86 & 72 & 74 & 71 & 73 & 96 & 80 & 86 & 92 \\
\texttt{shake\_bottle} & 98 & 100 & 100 & 96 & 99 & 99 & 96 & 96 & 100 & 100 \\
\texttt{shake\_bottle\_horizontally} & 100 & 100 & 98 & 94 & 100 & 99 & 100 & 100 & 100 & 100 \\
\texttt{stack\_blocks\_three} & 100 & 96 & 84 & 72 & 91 & 67 & 100 & 88 & 100 & 94 \\
\texttt{stack\_blocks\_two} & 100 & 100 & 96 & 90 & 100 & 92 & 98 & 98 & 98 & 100 \\
\texttt{stack\_bowls\_three} & 68 & 82 & 64 & 54 & 62 & 56 & 70 & 80 & 66 & 80 \\
\texttt{stack\_bowls\_two} & 96 & 92 & 92 & 82 & 89 & 90 & 92 & 92 & 96 & 88 \\
\texttt{stamp\_seal} & 98 & 90 & 84 & 60 & 84 & 47 & 82 & 68 & 98 & 94 \\
\texttt{turn\_switch} & 54 & 60 & 86 & 68 & 67 & 60 & 60 & 72 & 38 & 62 \\
\midrule
Average & 90.84 & 90.20 & 81.68 & 72.24 & 80.82 & 74.04 & 85.84 & 83.12 & \textbf{90.20} & \textbf{88.92} \\
\bottomrule
\end{tabular}

\end{table}

\end{document}